\documentclass[sigconf,nonacm]{acmart}

\usepackage{multirow}
\usepackage[cp1250]{inputenc}
\usepackage[T1]{fontenc}
\usepackage{calligra}
\usepackage{layout}
\usepackage{enumitem}

\usepackage{url}
\usepackage{color}
\usepackage{graphicx}
\usepackage{algorithm}
\usepackage{verbatim}
\usepackage{thmtools}
\usepackage{thm-restate}

\usepackage{hyperref}
\usepackage{nicefrac}
\usepackage{adjustbox}
\usepackage{subcaption}
\hypersetup{
colorlinks=false,
urlcolor=magenta           
}

\newcommand{\norm}[1]{\left\| #1 \right\|}
\newcommand{\sqnorm}[1]{\left\| #1 \right\|^2}
\newcommand{\inner}[2]{\left\langle #1, #2 \right\rangle} 
\newcommand{\roundbrack}[1]{\left(#1\right)}

\newcommand{\nlay}{\ell} 

\newcommand{\cC}{\mathcal{C}}
\newcommand{\cD}{\mathcal{D}}

\newcommand{\cO}{\mathcal{O}}

\newcommand{\del}[1]{}

\newcommand{\R}{\mathbb{R}} 

\newcommand{\eqdef}{:=}

\newcommand{\Exp}[1]{{\rm E}\left[#1\right]}

\newtheorem{assumption}{Assumption}
\newtheorem{lemma}{Lemma}

\newtheorem{theorem}{Theorem}

\theoremstyle{plain}

\theoremstyle{definition}

\usepackage[T1]{fontenc}    
\usepackage{hyperref}       
\usepackage{url}            
\usepackage{booktabs}       
\usepackage{nicefrac}       
\usepackage{microtype}      
\usepackage{xcolor}         
\usepackage{xspace}
\usepackage{algorithm}
\usepackage{algpseudocode}
\renewcommand{\algorithmiccomment}[1]{\bgroup\hfill//~#1\egroup}

\usepackage{caption}
\usepackage{graphicx}
\usepackage{pifont}
\usepackage{makecell}
\usepackage{threeparttable,adjustbox}
\usepackage{listings}
\definecolor{mydarkgreen}{RGB}{39,130,67}

\definecolor{mydarkred}{RGB}{192,47,25}

\definecolor{myteal}{RGB}{27,158,119}
\definecolor{myorange}{RGB}{217,95,2}
\definecolor{myred}{RGB}{231,41,138}
\definecolor{mypurple}{RGB}{152,78,163}
\definecolor{myblue}{RGB}{55,126,184}
\definecolor{mygreen}{RGB}{0,100,0}
\definecolor{mydarkred}{RGB}{192,47,25}
\definecolor{sgd}{rgb}{0., 0.2, 0.75}
\definecolor{quant}{rgb}{0.55, 0., 0.}

\newcommand{\smartparagraph}[1]{\noindent{\bf #1}\ }

\newcommand{\ExpD}[2]{{\rm E}_{#1}\left[#2\right]}
\definecolor{LightGray}{gray}{0.9}

\usepackage[colorinlistoftodos,bordercolor=orange,backgroundcolor=orange!20,linecolor=orange,textsize=scriptsize]{todonotes}

\graphicspath{{./image/}}



\makeatletter
\gdef\@copyrightpermission{
  \begin{minipage}{0.3\columnwidth}
   \href{https://creativecommons.org/licenses/by/4.0/}{\includegraphics[width=0.90\textwidth]{image/4ACM-CC-by-88x31.eps}}
  \end{minipage}\hfill
  \begin{minipage}{0.7\columnwidth}
   \href{https://creativecommons.org/licenses/by/4.0/}{This work is licensed under a Creative Commons Attribution International 4.0 License.}
  \end{minipage}
  \vspace{5pt}
}
\makeatother

\begin{document}

\title[Kimad: Adaptive Gradient Compression with Bandwidth Awareness]{Kimad: Adaptive Gradient Compression\\ with Bandwidth Awareness}

\author{Jihao Xin}
\authornote{Authors contributed equally to this work.}
\affiliation{%
  \institution{KAUST}
  \country{}
}

\author{Ivan Ilin}
\authornotemark[1]
\affiliation{%
  \institution{KAUST}
  \country{}
}

\author{Shunkang Zhang}
\affiliation{%
  \institution{HKUST}
  \country{}
}

\author{Marco Canini}
\affiliation{%
  \institution{KAUST}
  \country{}
}

\author{Peter Richt\'{a}rik}
\affiliation{%
  \institution{KAUST}
  \country{}
}

\begin{abstract}
In distributed training, communication often emerges as a bottleneck. In response, we introduce Kimad, a solution that offers adaptive gradient compression. By consistently monitoring bandwidth, Kimad refines compression ratios to match specific neural network layer requirements. Our exhaustive tests and proofs confirm Kimad's outstanding performance, establishing it as a benchmark in adaptive compression for distributed deep learning.

\end{abstract}

\maketitle

\section{Introduction}
Deep learning has steadily emerged as a transformative paradigm, demonstrating profound results in various domains. With its growth, there's been an explosion in the size of models and datasets. This upsurge in complexity often demands expansive computational resources, prompting researchers to adopt distributed training.

The Graphics Processing Unit (GPU) has emerged as a cornerstone in the realm of deep learning model training, fundamentally altering the landscape of artificial intelligence research and applications.  The latest advancements in GPU technology, exemplified by the state-of-the-art models such as the Ampere and Hopper~\cite{choquette2021nvidia, choquette2023nvidia}, exhibit unprecedented computational power and speed up the training by up to 16$\times$. However, it is noteworthy that the acquisition of these cutting-edge GPUs comes at a considerable financial cost, more than \$200,000 for a single DGX A100.

In this scenario, researchers increasingly turn to cloud-based computational resources for model training due to their flexible pricing models, variety of hardware, and ease of scaling computational resources. However, the bandwidth variability problem in cloud-based deep learning training poses a substantial challenge to the efficient execution of large-scale machine learning tasks~\cite{luo2020plink, shieh2011sharing, abdelmoniem2021dc2}. The bandwidth fluctuations, influenced by factors such as network congestion and competing workloads, lead to inconsistent performance during training. Figure \ref{fig:AWS} shows an example of bandwidth discrepancy measured at AWS EC2 with a TCP server in Frankfurt receiving simultaneously from 4 workers using IPerf3. While the existing framework CGX~\cite{markov2022cgx} has made strides by offering a comprehensive approach that integrates widely adopted gradient compression techniques and strikes a balance between accuracy and compression ratio, it failed to address dynamic bandwidth considerations. DC2~\cite{abdelmoniem2021dc2} achieves adaptive compression by inserting a shim layer between the ML framework and network stack to do real-time bandwidth monitoring and adjust the compression ratio. However, this approach is model-agnostic which cannot be used together with other application-level optimization.
\begin{figure}[t]
    \includegraphics[width=\linewidth]{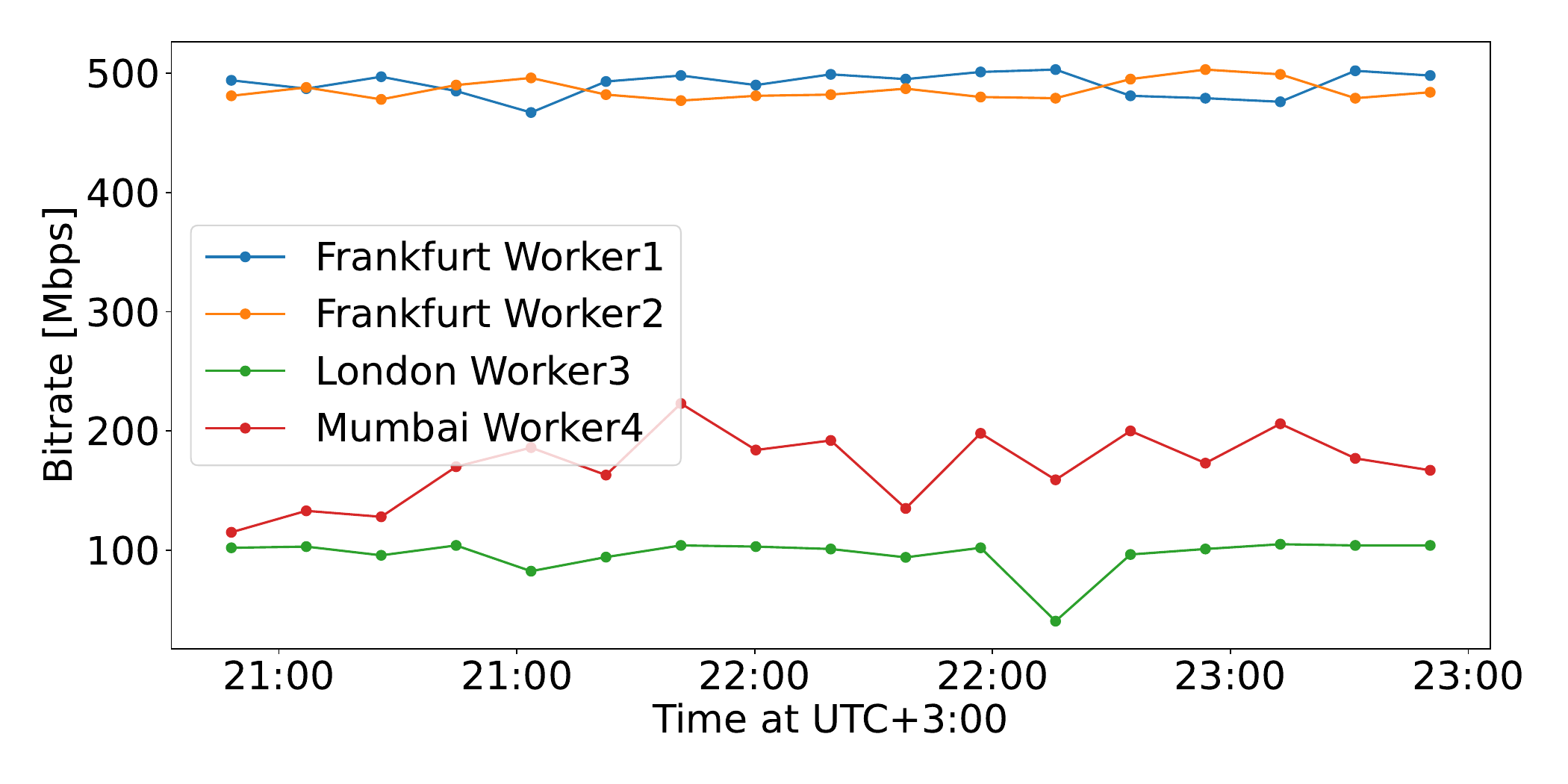}
    \captionsetup{aboveskip=0pt}
    \caption{EC2 bandwidth with TCP server at Frankfurt.}
    \label{fig:AWS}
\end{figure}
In addition to bandwidth adaptivity, numerous researchers are investigating how to capitalize on the diverse structural nature present across network layers in order to enhance compression ratios~\cite{lgreco, chen2018constraint}. However, these studies solely address the static nature of network structures and assume an ideally stable network connection, a scenario seldom encountered in real-world deployments.

In light of these findings, we introduce Kimad: an adaptive gradient compression system designed to be aware of both bandwidth changes and model structures. The comprehensive designs are depicted in Figure~\ref{fig:Kimad}. Kimad deploys a runtime bandwidth monitor and a compression module on each worker and server. Throughout the training phase, the bandwidth monitor gauges communication delays using historical statistics. Subsequently, the compression module utilizes the estimated bandwidth to compute the compression budget for the entire model. It then refines the layer-wise compression ratios while adhering to the overarching compression budget constraints.

In essence, we advance the following contributions:
\begin{itemize}[left=0pt,labelsep=10pt]
    \item We propose Kimad, a general framework for gradient compression adaptive to bandwidth changes.
    \item We further propose Kimad+, which incorporates an adaptive compression ratio tailored to layers to minimize compression errors.
    \item We expand the error feedback framework to elucidate its functionality within the context of Kimad.
    \item We evaluate from a convex function to a deep model, revealing that Kimad can accelerate the training while maintaining convergence.
\end{itemize}

\section{Background and Related Work}
\subsection{Data Parallelism}
Data parallelism is a widely used strategy to solve the distributed training problem, which can be formulated as \eqref{eq:main}.
\begin{equation}
\min_{x \in \R^{d}}  \left\{f(x) \eqdef \sum_{m=1}^{M} w_m f_m(x) \right\} \label{eq:main}
\end{equation}
$x\in \R^{d}$ corresponds to the parameters of a model, $[M]\eqdef \{1,\dots, M\}$ is the set of workers (e.g. GPUs, IoT devices) and $w_1, \dots, w_M$ are non-negative weights adding up to $1$ (for example,  the weights can be uniform, i.e., $w_m=\frac{1}{M}$ for all $m$). Further, $f_m(x)\eqdef \ExpD{\xi\sim \cD_m}{\ell(x,\xi)}$ is the empirical loss of model $x$ over the training data $\cD_m$ stored on worker $m$, where $\ell(x,\xi)$ is the loss of model $x$ on a single data point $\xi$.

In data parallelism, each worker keeps a copy of the model and a partition of the dataset. The gradients computed on each worker are then communicated to aggregate and update the model. We provide the general formulation to solve \eqref{eq:main}
in Appendix~\ref{app:formulation}.

In this work, we predominantly focus on the Parameter-Server model. Our choice is driven by its inherent capability to efficiently handle sparse updates \cite{omnireduce,parallax}, and its widespread adoption in environments with shared bandwidth like federated learning. While our emphasis is on the PS architecture with Data Parallelism, we posit that the adaptivity innovations we introduce can seamlessly integrate and offer value to the Peer-to-Peer architecture and model parallelism as well.

\subsection{Gradient Compression}
Relying on the nature that deep learning training can converge despite lossy information, gradient compression is a popular approach to speed up data-parallel training \cite{grace}. During back-propagation, gradients will be compressed before communication with the server, and the server will decompress the gradients prior to aggregating them; thus the communication cost can be largely reduced. Additionally, the server can distribute the model using compression as well. Gradient compression techniques can be generally categorized into three classes:
\begin{itemize}[left=0pt,labelsep=10pt]
    \item \textbf{Sparsification}~\cite{Suresh2017, RDME, Alistarh-EF-NIPS2018, stich2018sparsified, wang2018atomo}: Selectively retaining elements in gradients while zeroing others. This includes methods like Top$K$ (selecting the $K$ largest absolute value elements) and Rand$K$ (randomly selecting $K$ elements).
    \item \textbf{Quantization}~\cite{1bit, qsgd2017neurips, terngrad, mishchenko2019distributed, horvath2019natural}: Reducing data precision to fewer discrete values. Deep learning frameworks often use Floating Point 32 (FP32) for gradients, which can be compressed to formats like FP16, UINT8, or even 1 bit \cite{1bit}.
    \item \textbf{Low-Rank Decomposition}~\cite{vogels2019powersgd}: Approximating gradients by breaking them down into lower-rank matrices, reducing their size as $A \approx U \cdot V^T$, where $A$ is the original matrix, and $U$ and $V$ are lower-rank matrices.
\end{itemize}

\smartparagraph{Adaptive compression.}
Adaptive compression is an emerging area to study how to apply gradient compression efficiently with different compression levels \cite{accordion,abdelmoniem2021dc2}.  Gradient compression is traditionally used in an intuitive way: Given a compressor $\cC:\R^d\to\R^d$, gradients are compressed with a static strategy where the same compression ratio is used for each layer and across the whole training procedure. However, gradient compression has a different impact on different training stages. For instance, Accordion~\cite{accordion} selects between high and low compression levels by identifying the critical learning regimes. Furthermore, it is incumbent upon gradient compression methodologies to account for the diverse attributes of individual layers. For example, Egeria~\cite{egeria} methodically freezes layers that have achieved convergence during training. L-Greco~\cite{lgreco} uses dynamic programming to adjust layer-specific compression ratios given the error budget, reducing overall compressed size. Moreover, researchers should also take the system architecture into consideration. Notably, FlexReduce~\cite{FlexReduce} proposes that the communication protocol can be split into different portions unevenly based on the communication hierarchy.

\subsection{Error Feedback}
Error feedback (EF), also referred to as error compensation, is a widely adopted method for ensuring convergence stability in distributed training of supervised machine learning models. It is particularly effective when combined with biased compressors such as Top$K$.  EF was originally introduced as a heuristic~\cite{1bit}; then theoretical guarantees were proposed ~\cite{stich2018sparsified, Alistarh-EF-NIPS2018}. More recently, EF21~\cite{EF21,EF21BW, 3PC} provides theoretical analysis for distributed settings and achieves a state-of-the-art $\cO(1/T)$ convergence rate. We integrate EF21 into Kimad to achieve better convergence.

\subsection{Bandwidth Monitoring}
Bandwidth monitoring is critical in network management, especially in cloud-based scenarios. It addresses the need to monitor data transfer rates between computational nodes during training, ensuring optimal communication efficiency. Existing works~\cite{abdelmoniem2021dc2, caron2012auto, anand2012cloud} allow us to estimate the bandwidth changes by utilizing a collection of the network-level communication properties such as the latency. Particularly, adaptive strategies~\cite{xu2021bandwidth, wang2019adaptive} such as dynamic synchronization algorithms or buffering mechanisms can alleviate the effects of bandwidth fluctuation.

\section{Methodology}
We propose Kimad, an adaptive compression framework to accommodate varying bandwidth and model structures.
Kimad continuously monitors the bandwidth and dynamically adjusts the volume of communication size in each round for every machine. For instance, if the bandwidth $B_{m}^{k}$ for machine $m$ at step $k$ becomes limited compared to other devices, we instruct machine $m$ to employ a suitable compressor to reduce the size of the update vector\footnote{The update vector is the gradient under basic SGD setting. When applying EF21, it is the difference between the gradient estimation and the real gradient.} $u^k_m$  with the goal of ensuring that this machine does not become a straggler. Additionally, we present Kimad+, an extension of Kimad, which fine-tunes the compression ratio differently across layers. Kimad+ involves an additional step that introduces some computational overhead and is recommended for use when there is surplus computational capacity available (i.e., when communication is the most severe bottleneck).
\begin{figure*}[t]
  \centering
  \includegraphics[width=0.95\linewidth]{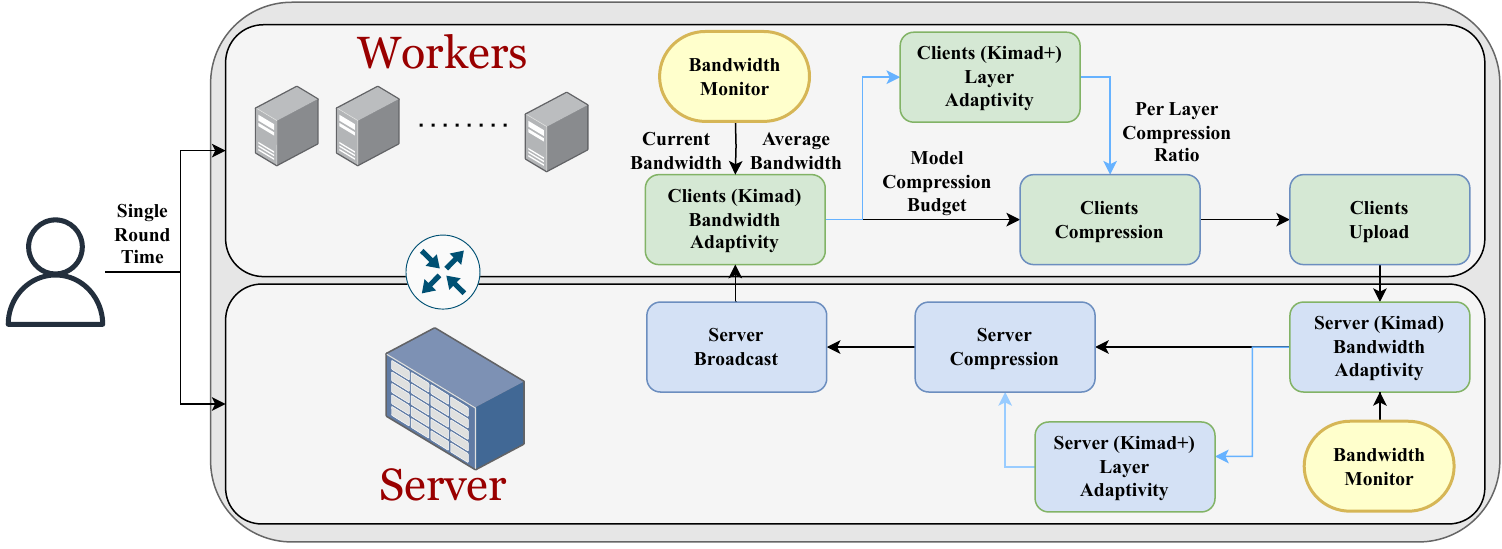}
  \caption{Kimad framework.}
  \label{fig:Kimad}
\end{figure*}

As Figure \ref{fig:Kimad} shows, to train a deep learning task, the end users need to inform Kimad of $t$, which is the \textbf{time budget} for a single communication round (a step).
The server and each worker will determine the compression strategy locally without knowing global information.
Kimad requires a bandwidth monitor, which is deployed on each worker and server, and will continuously monitor the network behavior and estimate the current bandwidth.
Kimad will calculate how many bits need to be communicated at each step based on the bandwidth, which we call the \textbf{compression budget} denoted as $c$. The blue arrows in Figure \ref{fig:Kimad} further represent Kimad+, which allocates the compression budget to each layer to minimize the compression error and thus improve accuracy.

Algorithm \ref{alg:Kimad} formulates the general version of Kimad.
The algorithm starts with the server broadcasting the latest compressed update $\cC^k( x^k - \hat{x}^{k-1})$, then each worker will calculate the update by $A_m^{\rm update}$ and upload the compressed update  $\cC_{m}^k(u_m^k-\hat{u}_m^{k-1} )$. Afterward, the server will update the model $x^k$ by the aggregated update vector.
The \textbf{core of the algorithm} is $A^{\rm compress}$, which selects a compressor from $\Omega$ in an {\bf adaptive} manner, based on the model information and current bandwidth estimation $B_m^k$. To recover accuracy, we apply bidirectional EF21, therefore, both server and workers maintain two estimators: $\hat{u}_m^k$  and $\hat{x}^k$, and only the server stores the global model $x^k$. We put a detailed version of the Kimad algorithm in Algorithm~\ref{alg:KimadFull} in Appendix \ref{app:kimad}.

\begin{algorithm}[h]
	\centering
		\caption{Kimad: Adaptive Gradient Compression with Bandwidth Awareness}
	\begin{algorithmic}[1]
		\State {\bf Input:} loss $\ell$; weights $w_m$; datasets $\cD_m$; model updating algorithm $A_m^{\rm update}$ on each worker $m\in [M]$; set of compressors $\Omega$; compressor-selection algorithm $A^{\rm compress}$ on server and workers; model $x^0\in \R^d$ known by the server; initial model estimator  $\hat{x}^{-1}\in \R^d$ and initial update estimators  $\hat{u}_{m}^{-1}\in \R^d$, known by the workers and the server;   single round time budget $t>0$; learning rate schedule $\{\gamma^k\} > 0$ for iterations $k\geq 0$.
	\For{{\bf each communication round} $k=0,1, 2, \dots $}

  \state //Server:
  \State Estimate $B^k$ at communication round $k$
  \State Select compressor: $\cC^k = A^{\rm compress}(\Omega, x^k, \hat{x}^{k-1}, B^k, t)$
  \State Update model estimator: $\hat{x}^{k} = \hat{x}^{k-1} + \cC^k( x^k - \hat{x}^{k-1})$
  \State Broadcasts the compressed vector $\cC^k( x^k - \hat{x}^{k-1})$ to all workers $m \in [M]$

\state //Workers:
\For{{\bf each worker} $m=1, 2, \dots, M $ {\bf in parallel}}			
            \State Update model estimator:
            \[\hat{x}^{k} = \hat{x}^{k-1} + \cC^k( x^k - \hat{x}^{k-1})\]
		\State Calculate update: $u_m^k = A_m^{\rm update} \left(\hat{x}^k, \ell, \cD_m \right) \in \R^d$
            \State Select compressor:
            \[\cC_m^k = A^{\rm compress}(\Omega, u_m^k, \hat{u}_m^{k-1}, B_{m}^k, t)\]
            \State  Upload the compressed vector $\cC_{m}^k(u_m^k-\hat{u}_m^{k-1} )$ to the server

		\EndFor
\State //Server:
\State Aggregate all update estimators: $$\hat{u}_m^k = \hat{u}_{m}^{k-1} + \cC_{m}^k(u_m^k-\hat{u}_m^{k-1} ), \qquad m\in [M]$$
\State Updates the model via
$x^{k+1} = x^k - \gamma^k \sum_{m=1}^M w_m \hat{u}_m^k$
		\EndFor
	\end{algorithmic}
\label{alg:Kimad}
\end{algorithm}

\subsection{Kimad: Bandwidth Adaptivity}
With a user-specified time budget $t$, the target of Kimad is to limit the training time at each step within $t$ time units while communicating as much information as possible.

In our work, we examine asymmetric networks, e.g., the up-link and down-link bandwidth can be different, and the bandwidth varies among workers. We apply bidirectional compression, i.e., both workers and server send compressed information.

We break down the time cost of worker $m$ at step $k$ as: $$t = T_{comm}^{u} + T_{comp} + T_{comm}^{d}$$
We abstract the computation time of a step as $T_{comp}$ which is assumed to be constant across a training task.
For the uplink communication, we define $T_{comm}^{u} = \frac{c}{B_{m}^{k}}$. For the downlink, we define $T_{comm}^{d} = \alpha \frac{c}{B_{m}^{k}}$ and $\alpha$ is the coefficient of broadcasting congestion which can be simply set to 1 assuming no congestion. Therefore, for simplicity, and without loss of generality, we only consider varying $T_{comm}$ to simulate various scenarios.

When communication is triggered, Kimad will read the current bandwidth from the bandwidth estimator and use it to calculate (with negligible computation overhead) the compression budget as:
\begin{equation}
    c = B_{m}^{k} \frac{t-T_{comp}}{2}.
    \label{eq:budget}
\end{equation}

\subsection{Kimad+: Layer Adaptivity}
With a predefined compression budget, Kimad+ can dynamically allocate the compression ratios of individual layers in a non-uniform manner. This optimization aims to enhance performance while ensuring that the cumulative compression ratio remains within the allocated budget. We start by formulating it as an optimization problem as:
\begin{equation}
\label{eq:optimization}
\min \varepsilon^{k} = \sum_{i=1}^{l}\varepsilon^{i},\quad \text{subject to} \sum_{i=1}^{\nlay} b_{i j_i^k}\leq c
\end{equation}
The target is to minimize the total error $\varepsilon^{k}$ across layers caused by compression, with compressed size constrained by the compression budget. We consider the standard Euclidean norm ($l_2$-norm) as the error indicator defined by:
\begin{equation}\label{eq:norm_decomposition}
\varepsilon^{k} = \norm{\hat{u}^k - u^k}^2 = \sum_{i=1}^{\nlay} \norm{\hat{u}^k_i - u_i}^2
\end{equation}

However, the relation between the compression error and compressed size is not deterministic, and the search space of the compression ratio is continuous. As a result, finding an analytical solution for this optimization problem is not feasible. To tackle this challenge, we employ a discretization approach, narrowing down the compression ratio search space. Specifically, for each layer, Kimad+ restricts its choice of compression ratio to a discrete set $\{1, 2, \ldots, w\}$. Therefore, \eqref{eq:optimization} can be written as:
\begin{eqnarray*} \min\limits_{j_1^k, \cdots, j_{\nlay}^k} && \varepsilon^{k}(j_1^k,\dots,j_{\nlay}^k) = \sum_{i=1}^{l}\varepsilon^{i}(j_i^k) \\
\text{subject to}  &&  j_1^k\in \{1,\dots,w_1\}, \cdots, j_{\nlay}^k\in \{1,\dots,w_{\nlay}\}\\
&& \sum_{i=1}^{\nlay} b_{i j_i^k}\leq c,
\end{eqnarray*}

We adopt the idea from L-Greco~\cite{lgreco} to formulate it as a knapsack problem. In contrast to L-Greco, Kimad+ uses the compression budget $c$ as the knapsack size and the compression error as the weight. Then, Kimad+ uses dynamic programming to solve the knapsack problem. The time complexity is $O(NKD)$ where $N$ is the number of layers, $K$ represents the possible compression ratios, and $D$ is the discretization factor for the error. We give the algorithm details in Appendix \ref{app:dp}.

\begin{figure*}[t!]
    \centering
    \begin{minipage}[b]{0.49\linewidth}
        \includegraphics[width=0.49\linewidth]{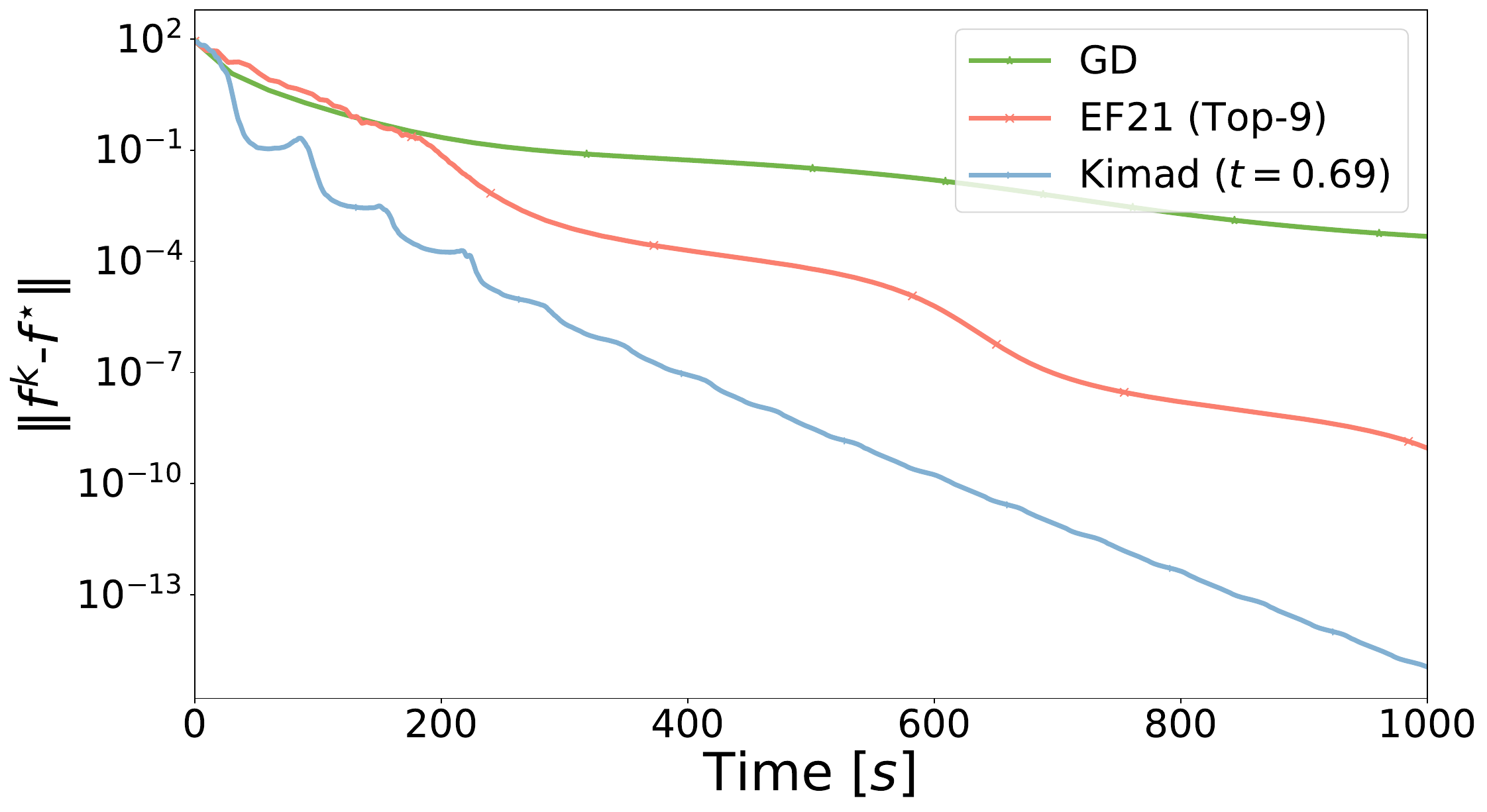}
        \includegraphics[width=0.49\linewidth]{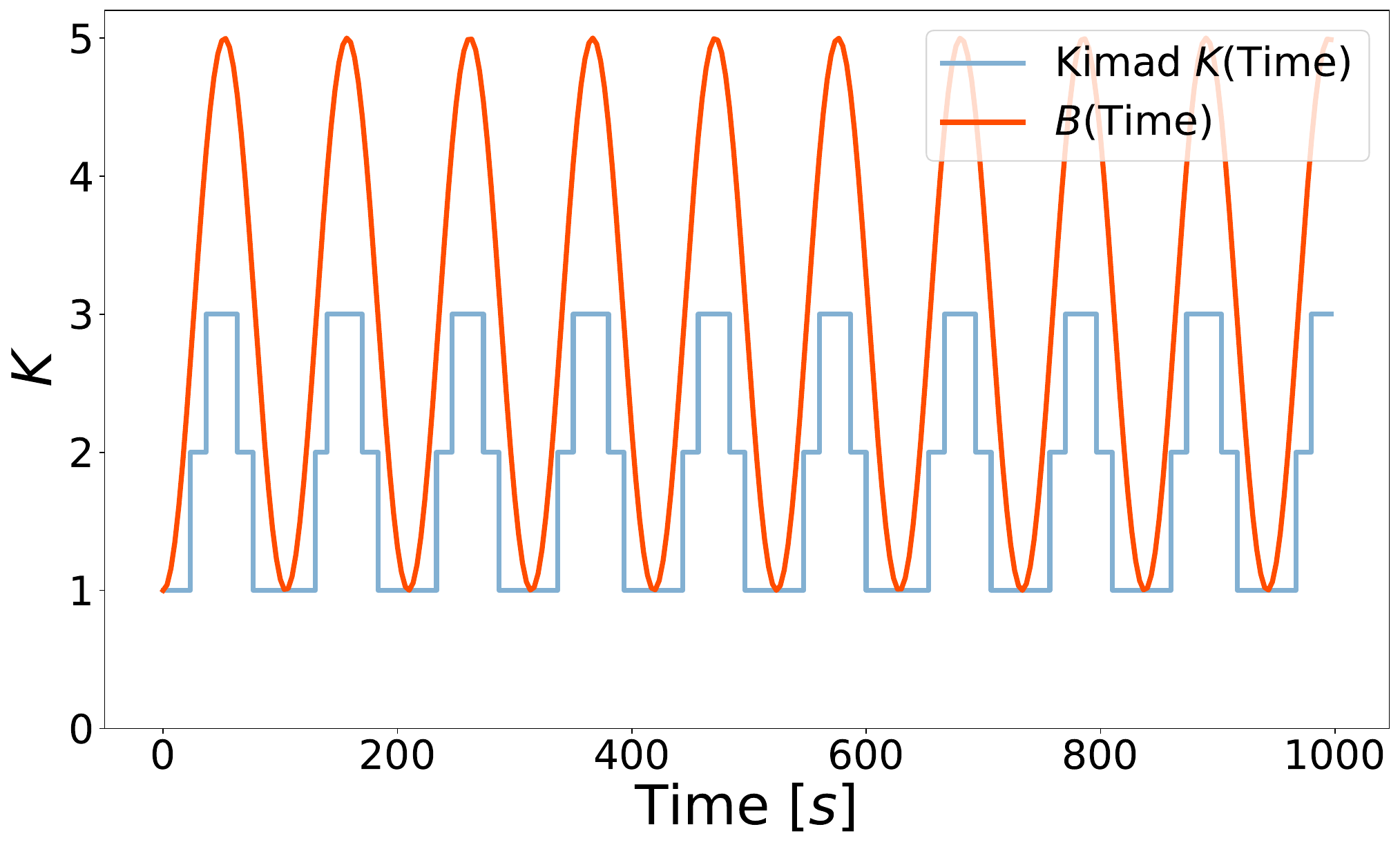}
        \captionsetup{width=.9\linewidth}
        \captionsetup{aboveskip=0pt}
        \caption{Extremely small bandwidth: $B_{max} << d$.}
        \label{fig:quad_1}
    \end{minipage}
    \begin{minipage}[b]{0.49\linewidth}
        \includegraphics[width=0.49\linewidth]{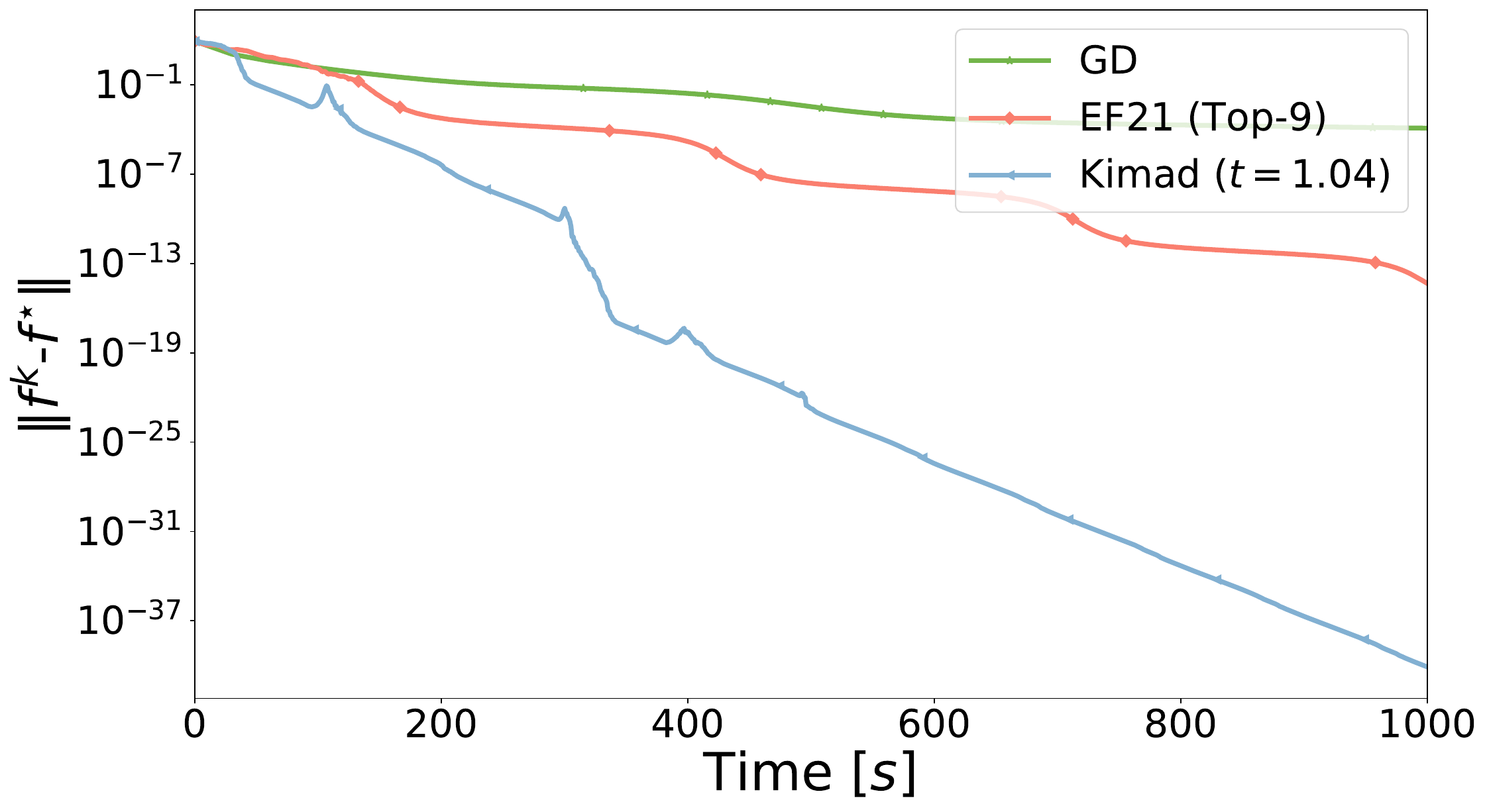}
        \includegraphics[width=0.49\linewidth]{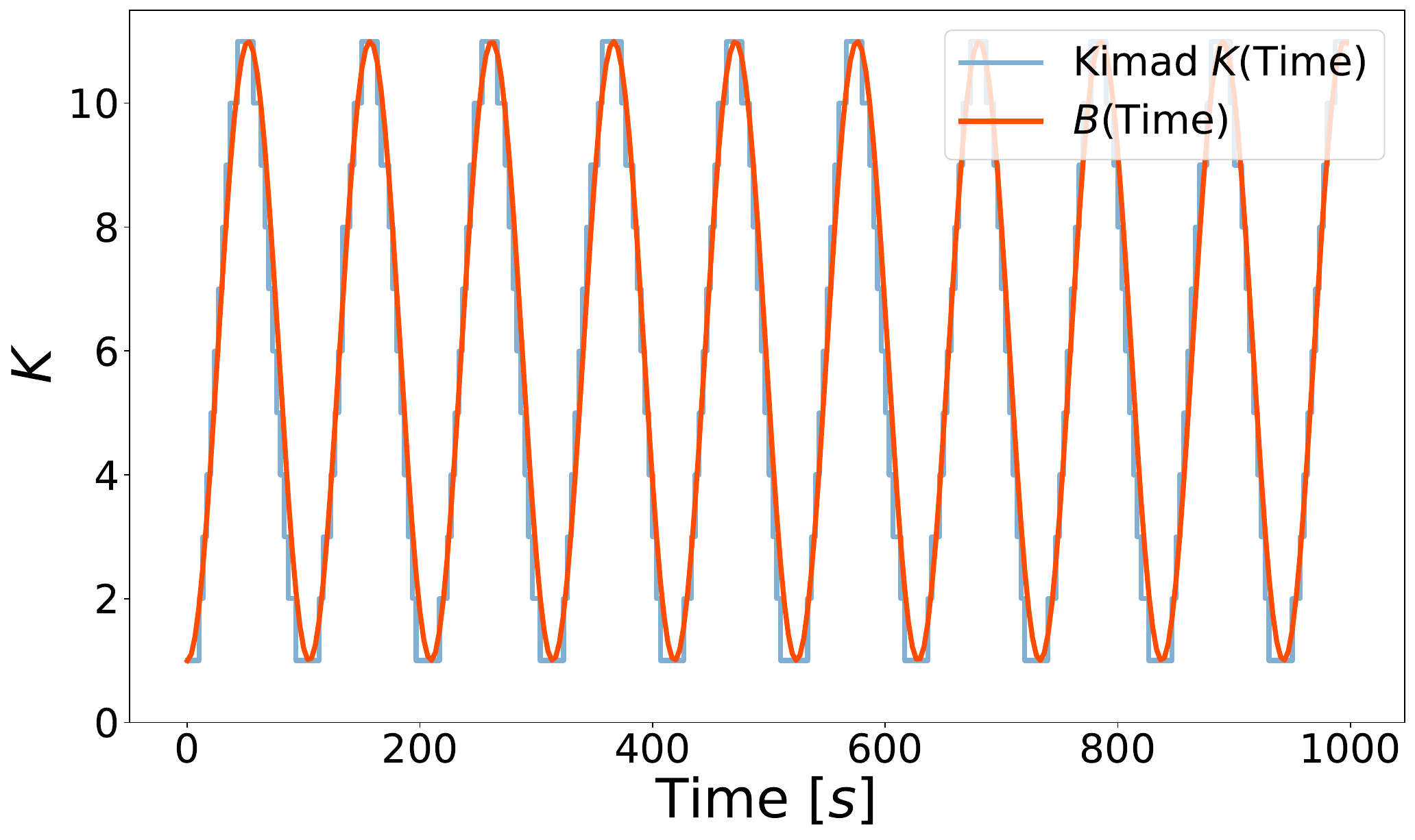}
        \captionsetup{width=.8\linewidth}
        \captionsetup{aboveskip=0pt}
        \caption{Small bandwidth: $B_{max} < d$.}
        \label{fig:quad_2}
    \end{minipage}
    \begin{minipage}[b]{0.49\linewidth}
        \includegraphics[width=0.49\linewidth]{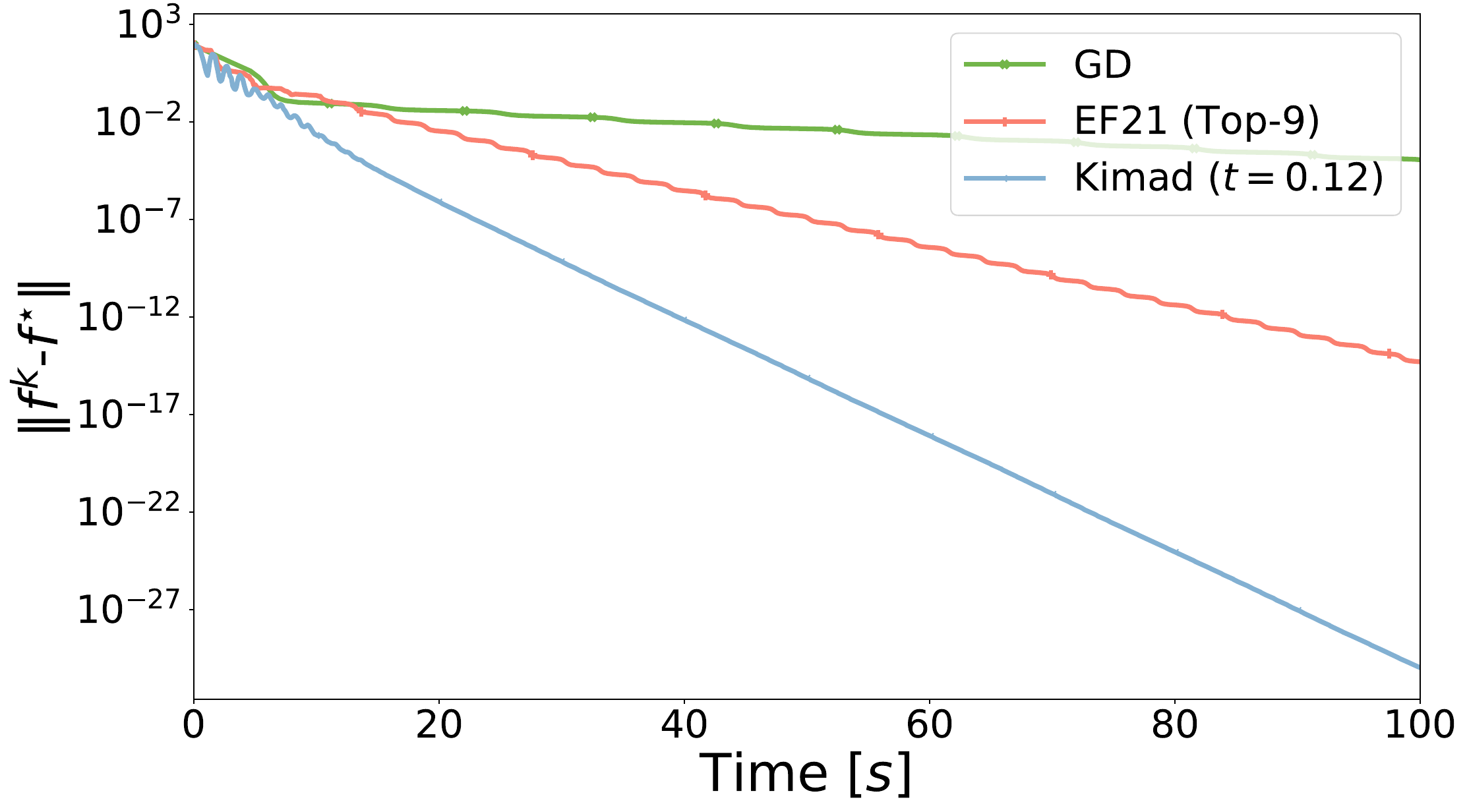}
        \includegraphics[width=0.49\linewidth]{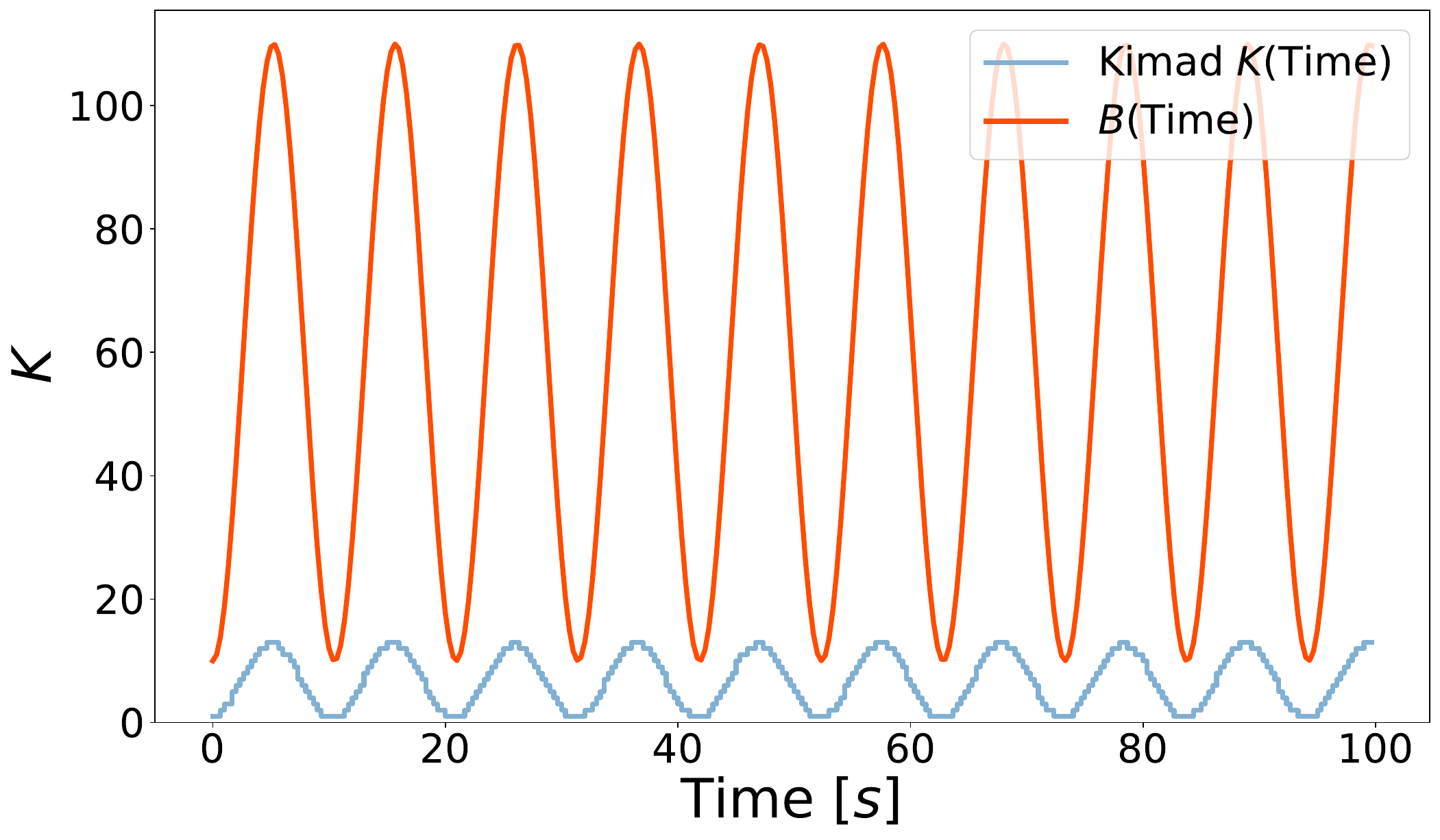}
        \captionsetup{width=.8\linewidth}
        \captionsetup{aboveskip=0pt}
        \caption{Oscillation between small and high bandwidth.}
        \label{fig:quad_3}
    \end{minipage}
    \begin{minipage}[b]{0.49\linewidth}
        \includegraphics[width=0.49\linewidth]{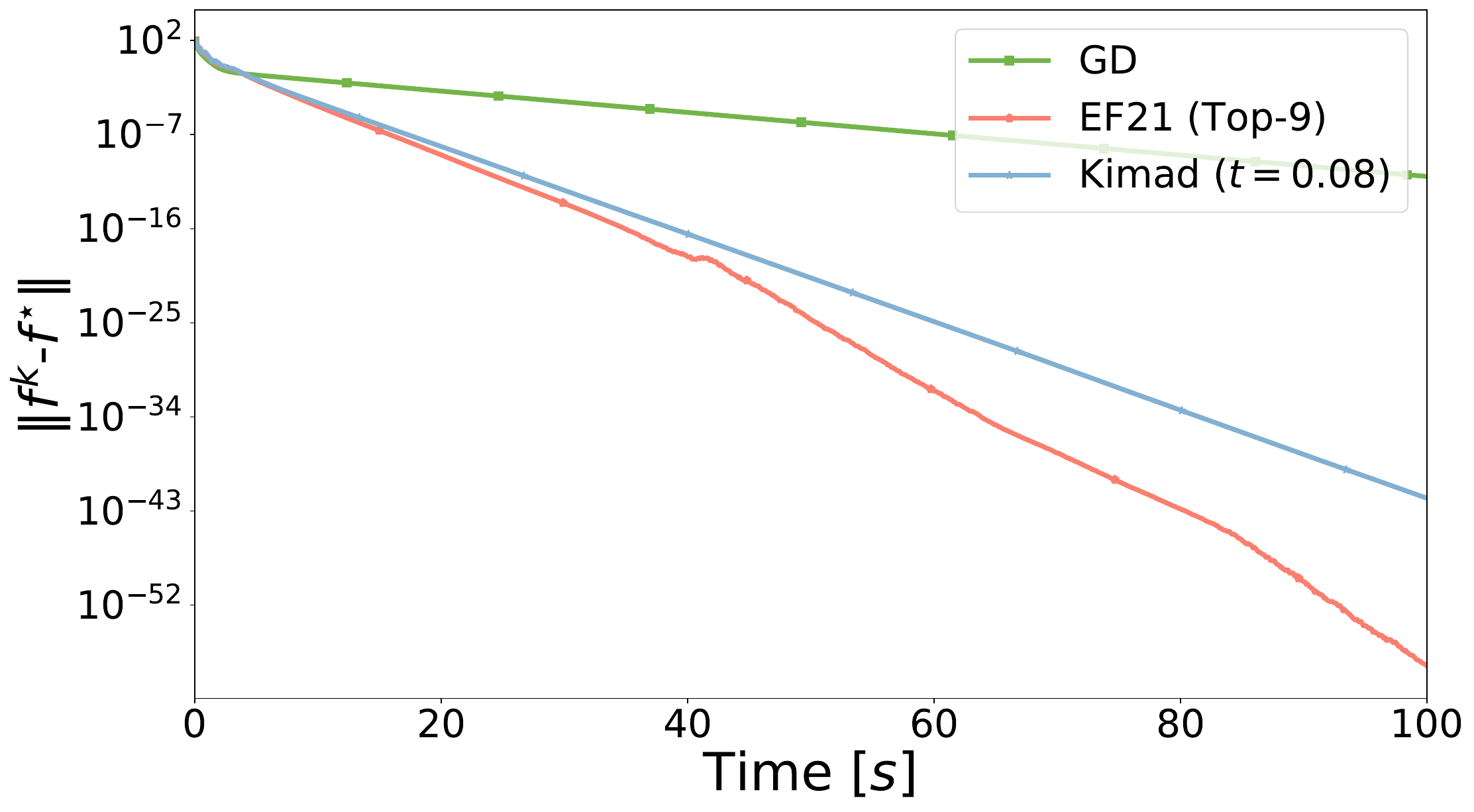}
        \includegraphics[width=0.49\linewidth]{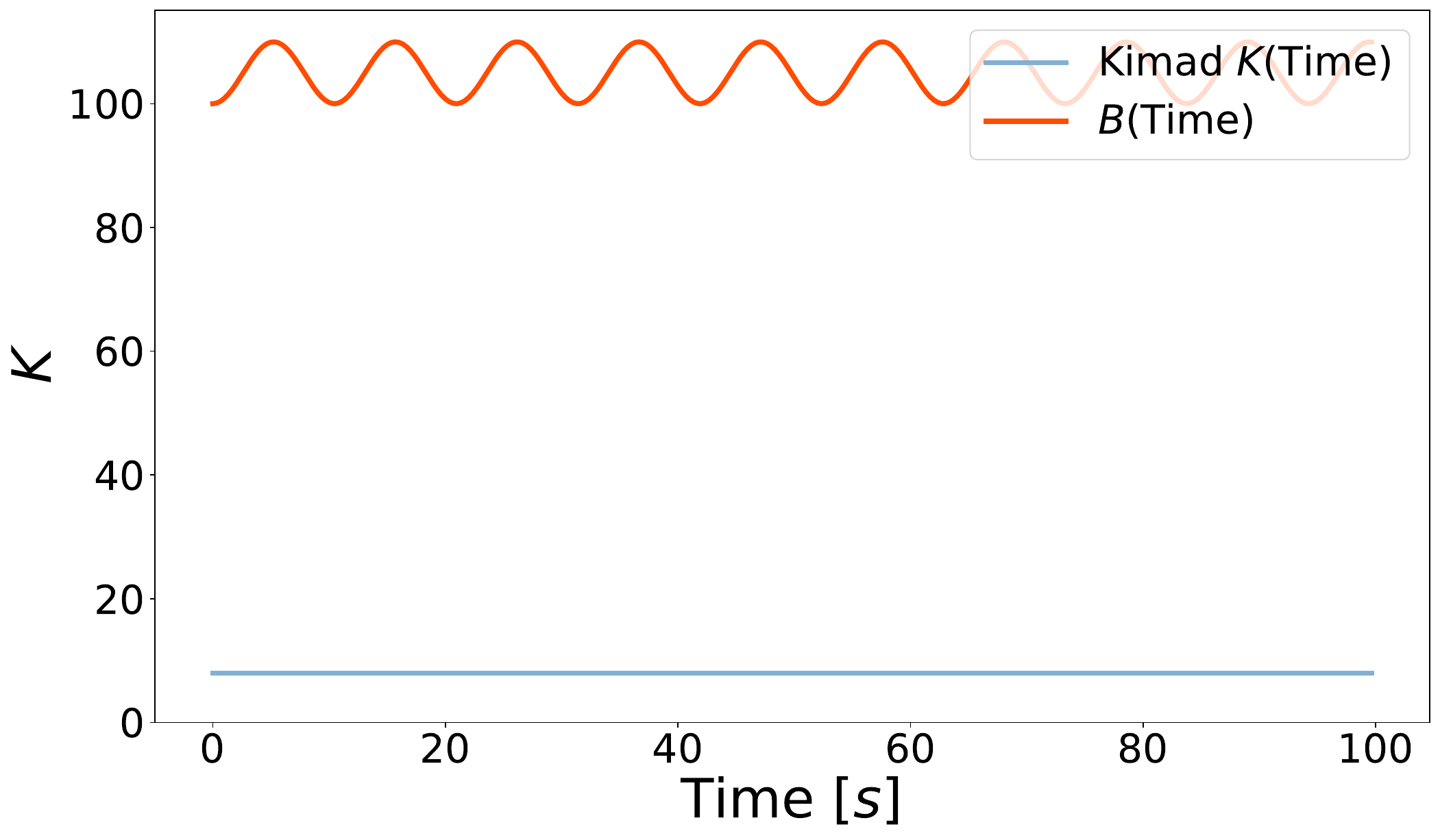}
        \captionsetup{width=.8\linewidth}
        \captionsetup{aboveskip=0pt}
        \caption{High bandwidth with small oscillations (almost no adaptation).}
        \label{fig:quad_4}
    \end{minipage}
\end{figure*}

\subsection{Error Feedback}
We apply error feedback within Kimad. To the best of our knowledge, EF21~\cite{EF21} is one of the most effective EF methods. We adapt EF21 and extend it in a layer-wise fashion. However, while EF21 is analyzed using a constant step size, our theory here allows the step size to depend on the layer $i$ and on the iteration $k$. Below we give the theoretical result; the proof is in Appendices \ref{app:assumptions}, \ref{app:lemma}, and \ref{app:proof}.

We initialize by choosing $x_i^0,\hat{u}_i^0\in \R^{d_i}$ for $i=1,\dots,\nlay$, $x^0=(x_1^0,\dots,x_{\nlay}^0)\in \R^d$. After this, for $k=0,1,\dots$ we iterate:
\begin{eqnarray}
x_i^{k+1} &=& x_i^k - \gamma_i^k \hat{u}_i^{k}, \quad i=1,\dots,\nlay,  \label{eq:EF21-1} \\
\hat{u}_i^{k+1} &=& \hat{u}_i^k + \cC_i^k (\nabla_i f(x^{k+1}) - \hat{u}_i^k ), \quad i=1,\dots,\nlay, \label{eq:EF21-2} \\
x^{k+1} &=& (x_1^{k+1},\dots,x_{\nlay}^{k+1}), \label{eq:EF21-3}
\end{eqnarray}
where  $\gamma_i^k\geq 0$ are step sizes and $\cC_i^k:\R^{d_i}\to \R^{d_i}$ are compressors for all $i=1,\dots,\nlay$ and $k\geq 0$. Let $\hat{u}^k = (\hat{u}_1^k,\dots, \hat{u}_{\nlay}^k)$.

\smartparagraph{Theory.} We state our main theorem below. We assume the model is layer-smooth (Appendix \ref{app:layer_smooth}) and the result extends to global-smooth (Appendix \ref{app:global_smooth}).

\begin{theorem} \label{thm:main} Consider Algorithm~\eqref{eq:EF21-1}--\eqref{eq:EF21-3}. Assume $f$ is lower bounded by $f^{\inf}\in \R$, and let the layer smoothness  (Assumption~\ref{ass:L_i}) and global smoothness (Assumption~\ref{ass:L}) conditions hold. Assume that $\cC_i^k \in \mathbb{C}^{d_i}(\alpha_i)$ for all $k\geq 0$ and  all $i \in [\nlay]$, where $\alpha_i \in (0,1]$. Choose any $\zeta_1,\dots,\zeta_{\nlay}>0$ such that $(1- \alpha_i )(1+\zeta_i) < 1$ for all $i$, and define
		\begin{equation}\label{eq:theta-beta-def-0}\theta_i \eqdef 1- (1- \alpha_i )(1+\zeta_i), \quad  \beta_i \eqdef (1- \alpha_i ) \left(1+ \zeta_i^{-1} \right)
		\end{equation}
and $\theta \eqdef \min_i \theta_i$. Choose any $w_1, \dots, w_{\nlay}>0$ and let  the step sizes be chosen via $\gamma_i^k \equiv \gamma w_i$ for all $k\geq 0$, where $\gamma >0$ satisfies
			\begin{equation} \label{eq:stepsize_main}\gamma^2 \frac{w_i \left(\max_i \frac{w_i}{\delta_i} \right)\left(\max_i \delta_i \beta_i  \right) L^2}{\theta} + \gamma L_i w_i \leq 1.\end{equation}

Then
\begin{align*}
    \frac{1}{K} \sum_{k=0}^{K-1}  & \left( \sum_{i=1}^{\nlay} w_i \Exp{\sqnorm{\nabla_i f(x^{k})}} \right) \leq \\
    \frac{2 (f(x^0) - f^{\inf})}{\gamma K} + &\frac{\left(\max_i \frac{w_i}{\delta_i} \right) \sum\limits_{i=1}^{\nlay} \delta_i  \sqnorm{\hat{u}_i^0 - \nabla_i f(x^{0})}}{\theta K}.
\end{align*}
\end{theorem}

\section{Evaluation}
We begin our evaluation by initially performing synthetic experiments to showcase the efficiency of our proposed Kimad method, particularly to demonstrate that EF21 can work with compression ratio adaptive to bandwidth. The synthetic experiments are done with a simple quadratic function $f$ which is lower bounded by $0$, and has layer smoothness  (Appendix \ref{app:layer_smooth}) and global smoothness (Appendix \ref{app:global_smooth}). This function fits the theory assumptions and allows us to fine-tune the learning rates for all compression ratios and time budget $t$ at an affordable cost.
Subsequently, we present results from more practical tasks, demonstrating that Kimad is applicable to distributed deep learning training. We also conduct an evaluation of Kimad+ to substantiate its superior capabilities of reducing compression error compared to Kimad, all while maintaining the same communication cost.
The evaluation is simulation-based, running as a Parameter Server architecture with dynamic asymmetric bandwidth. We use TopK with fixed K as the default compression method. The simulator is tested with Python 3.9.15, and Pytorch 1.13.1.
\subsection{Synthetic Experiments}
For now, we consider only one direction; e.g., the down-link (server to worker) communication cost can be neglected. So, there is only an up-link bandwidth cost. We simulate the bandwidth oscillation with a sinusoid-like function as Figure~\ref{fig:quad_1}.

We start our experiments in a single-worker setup to optimize a quadratic function. So, $M=1$, $w_1=1$, $\cD_1 = a = (a_1,\dots,a_d)$, $a_i > 0$ for all $i$ and $f_1(x)\eqdef \ExpD{\xi\sim \cD_m}{\ell(x,\xi)} = \ell(x,\xi) = \frac{1}{2}\sum_{i=1}^d{a_i x_i^2}$, where $d = 30$. Hence $f(x)$ in problem \eqref{eq:main} can be written in the following form:
\begin{equation}
\label{eq:simple_quad}
f(x) = \frac{1}{2}\sum_{i=1}^d{a_i x_i^2}
\end{equation}

Previous works~\cite{EF21, EF21BW, 3PC} show that EF21 can improve convergence rates in federated learning setups, particularly for biased compressors such as TopK. We now demonstrate that EF21 can also be used to improve performance seamlessly with Kimad.
For a fair comparison, it's crucial to fine-tune all hyperparameters for each method. For EF21 with TopK, we systematically explored various K values and selected the one that performed the best for comparison with Kimad. However, Kimad doesn't require us to determine the best K since it adapts to the available bandwidth dynamically. Instead, we focus on optimizing the time budget parameter $t$ and fine-tuning Kimad in conjunction with EF21. We compare performance among Kimad, EF21, and set the standard gradient descent (GD) as the baseline.

As Figure~\ref{fig:quad_1} shows, Kimad can be much faster than the best EF21. We achieved these results because Kimad adapted the compress ratio depending on the bandwidth to be as effective as possible. These results are consistent over different bandwidth patterns: with small bandwidth ($B_{max} < d$) and high relative oscillation we can see great results because we gain more with using adaptive strategy (Figure~\ref{fig:quad_1} and Figure~\ref{fig:quad_2}). As the amplitude of the bandwidth oscillations becomes higher, we still have improvements in performance (Figure~\ref{fig:quad_3}). However, when the bandwidth is very high and the amplitude of its oscillations is low, we do not gain from adapting of compress ratio: there is no need to adapt because the bandwidth is not a bottleneck anymore (Figure~\ref{fig:quad_4}).
\begin{figure*}[t]
    \centering
    \begin{minipage}[b]{0.33\linewidth}
        \includegraphics[width=\linewidth]{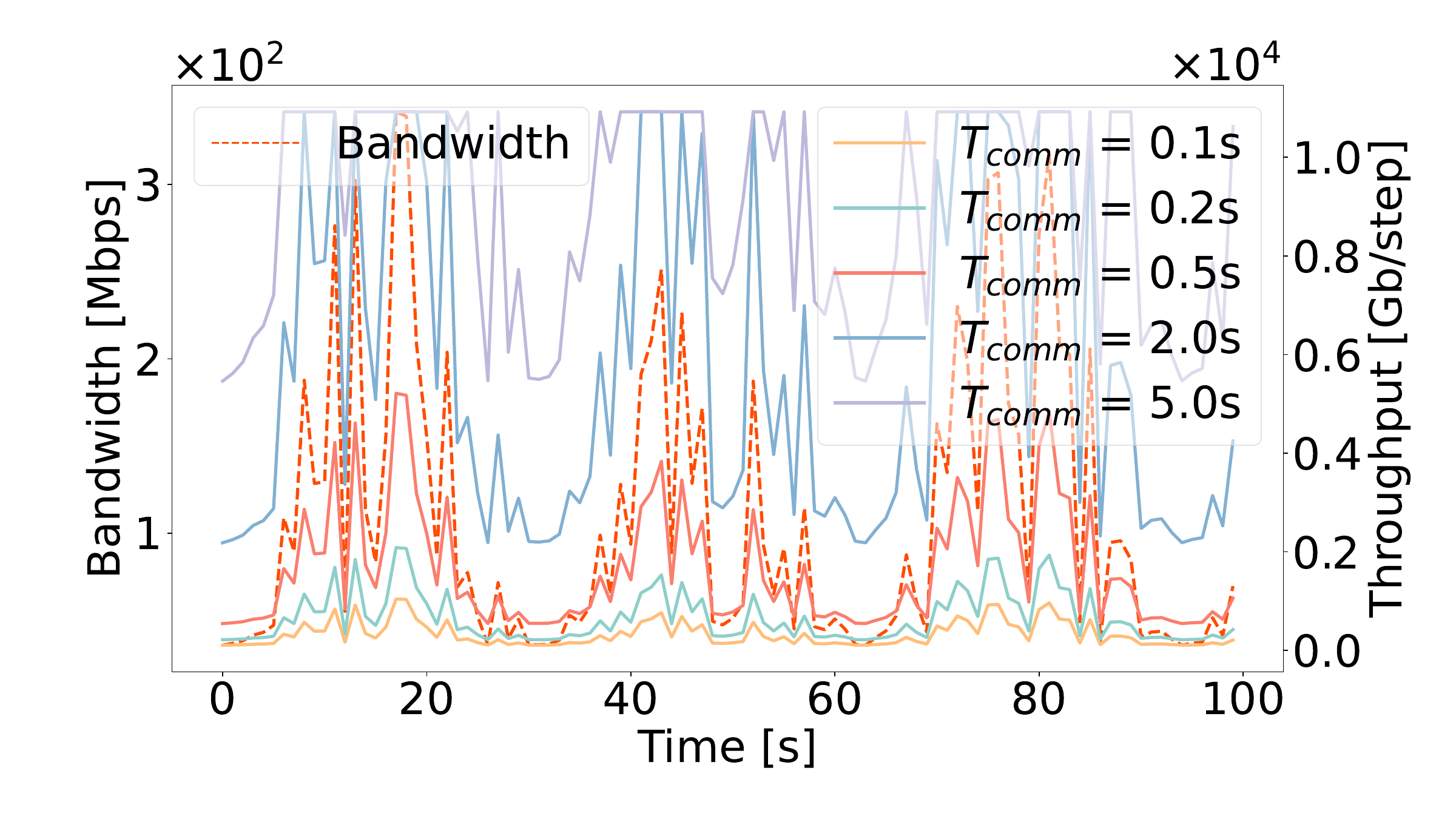}
        \captionsetup{width=.9\linewidth}
        \captionsetup{aboveskip=0pt}
        \caption{Communication throughput. $M=4$ workers.}
        \label{fig:ResNetBW}
    \end{minipage}
    \begin{minipage}[b]{0.33\linewidth}
        \includegraphics[width=\linewidth]{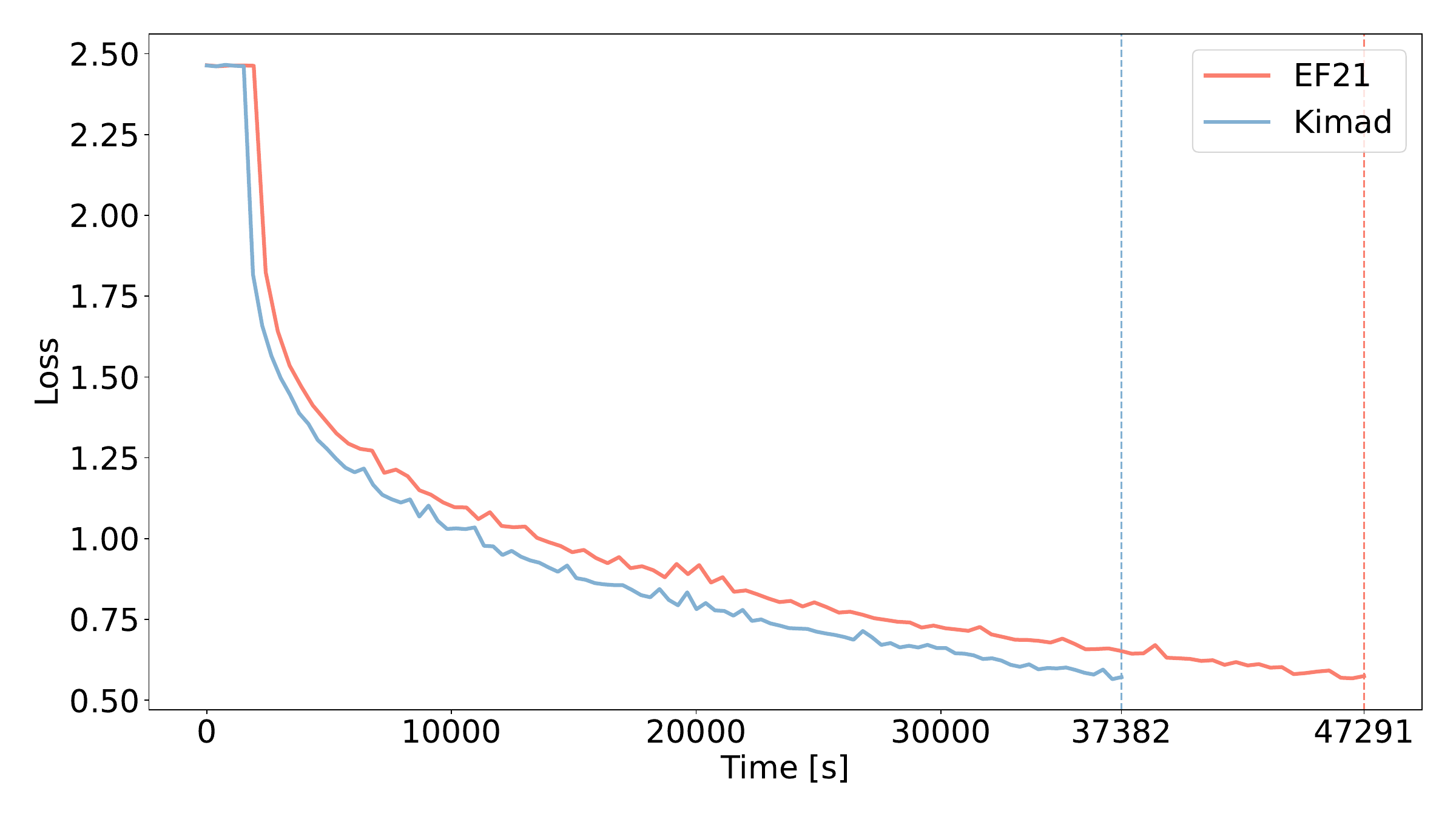}
        \captionsetup{width=.8\linewidth}
        \captionsetup{aboveskip=0pt}
        \caption{Loss curve. $M=4$ workers, $T_{comm}=1.0s$.}
        \label{fig:ResNetLoss}
    \end{minipage}
    \begin{minipage}[b]{0.33\linewidth}
        \includegraphics[width=\linewidth]{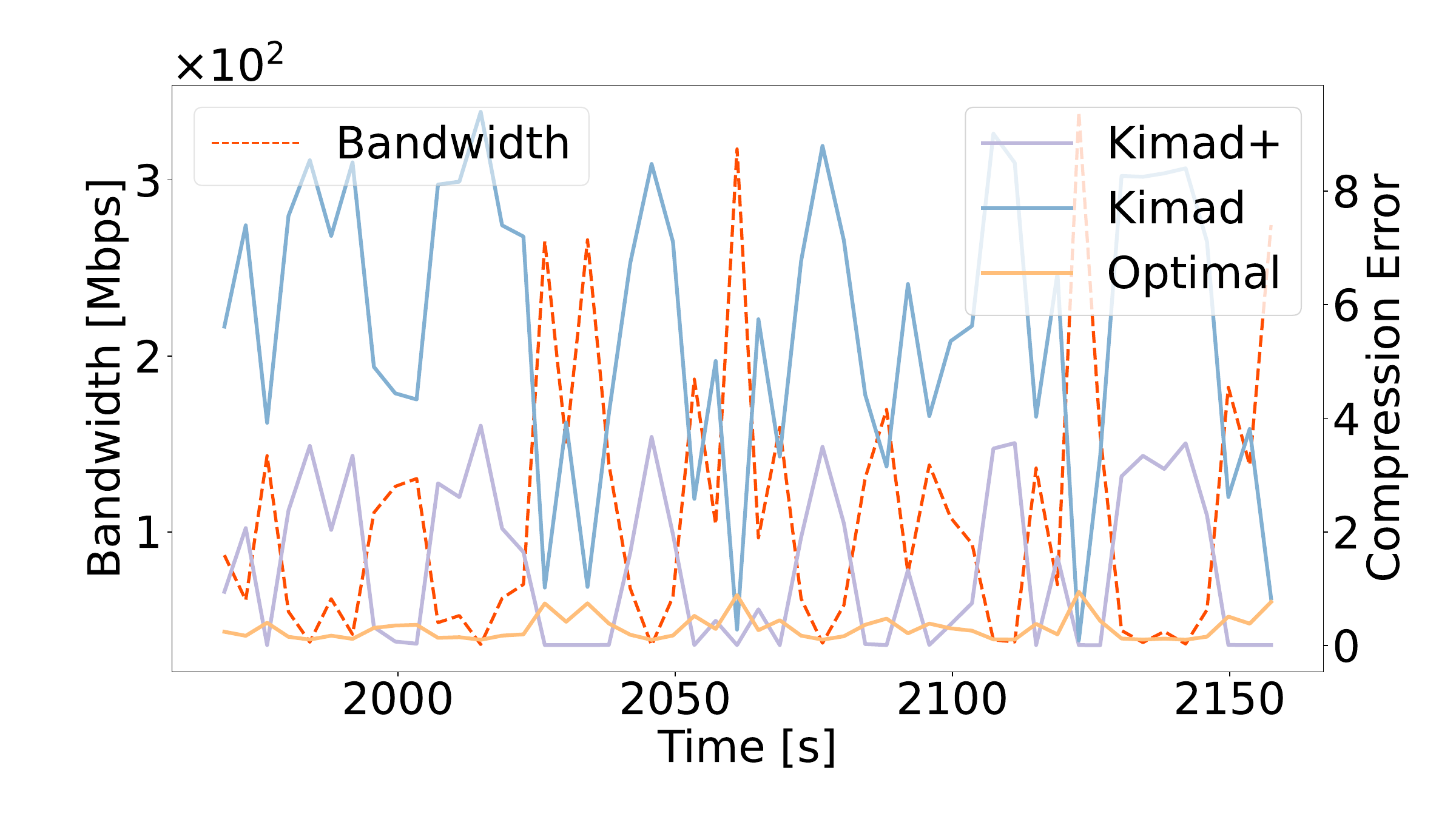}
        \captionsetup{width=0.9\linewidth}
        \captionsetup{aboveskip=0pt}
        \caption{Compression error. $M=4$ workers, $T_{comm}=1.0s$.}
        \label{fig:Kimad+}
    \end{minipage}
\end{figure*}

\subsection{Kimad on Deep Model}
\smartparagraph{Setting.} We train ResNet18 on Cifar10 for 100 epochs. We set $\{w_m = 1, \forall m \in M\}$, $\gamma=0.01$, $\Omega= \{TopK | K >0\}$, $A^{update}$ calculates gradients with batch size = 128, random seed=21. We conduct 5 epochs warmup training, thus $\hat{u}^k_m $ and $\hat{x^k}$ are initialized as $u^5_m$ and $x^5$ . Compression occurs on a per-layer basis, in accordance with common practice. We set $\alpha=1$ for the downlink so that the compression budget can be calculated by $ c = T_{comm} B_{m}^{k}$.
We set $T_{comp} = \frac{ModelSize}{AverageBandwidth}$.\footnote{Measured during the warmup epochs.} Our baseline is EF21 with fixed-ratio compression, which has the same overall communication size as Kimad but applies the same compression ratio across layers and steps.

\smartparagraph{Bandwidth.} In our simulation, we model the dynamic bandwidth within the range of 30 Mbps to 330 Mbps using the function: $Bandwidth(time) = \eta \sin (\theta \cdot time)^2 + \delta$, where $\eta$, $\theta$, $\delta$ are user-defined coefficients to adjust the changing frequency and amplitude.
We assume the bandwidth between the server and each worker follows the same patterns with different noise. The dashed curve in Figure \ref{fig:ResNetBW} shows the bandwidth pattern.

\smartparagraph{Communication adaptivity.}
Figure \ref{fig:ResNetBW} depicts a single worker's communication size over time for different T\textsubscript{comm}. The left y-axis represents bandwidth, while the right represents communication size. The plateau at the top signifies the maximum uncompressed size. This graph illustrates Kimad's effective adaptation to changing bandwidth conditions, thus optimizing communication throughout.

\smartparagraph{Convergence.}
The loss curve in Figure \ref{fig:ResNetLoss} shows the comparison with EF21. Kimad finishes training faster while achieving the same final convergence.
\begin{table}[t]
    \centering
    \begin{tabular}{|l|c|c|c|c|}
    \hline
     & \textbf{1.0s} & \textbf{0.5s} & \textbf{0.2s} & \textbf{0.1s} \\
    \hline
    EF21 & 486.1s & 360.6s & 284.2s & 258.0s \\
    \hline
    Kimad & 385.2s & 285.2s & 225.2s & 205.2s \\

    \hline
    \end{tabular}
    \caption{Average step time across T\textsubscript{comm}. $M=4$ workers.}
    \label{tb:steptime}
\end{table}

\smartparagraph{Speedup.}
Table \ref{tb:steptime} lists the average time of one SGD step across different T\textsubscript{comm}. In our setting, Kimad can generally save 20\% training time for different communication budgets.

\smartparagraph{Scalability.}
Table \ref{tb:Accuracy} presents the Top5 accuracy on the evaluation set after 100 epochs. Kimad demonstrates comparable scalability to EF21 which maintains good accuracy levels with increasing number of workers.

\subsection{Kimad+}
Kimad+ minimizes the compression error while maintaining the same compression ratio as Kimad. We train Kimad and Kimad+ under the same setting as above with error discretization factor 1000 and compression ratio chosen from $\{x \mid x = 0.01 + k \cdot 0.02, \text{ where } k \in \mathbb{Z}, 0.01 \leq x \leq 1\}$. Figure \ref{fig:Kimad+} shows the compression error at one worker in a time frame, the optimal baseline is to select K with the whole model information.
The compression error is negatively correlated with bandwidth, while Kimad+ can generally achieve lower compression error. We also observe that Kimad+ can achieve 1\% higher accuracy than EF21 after the training.
\begin{table}[t]
    \centering
    \begin{tabular}{|l|c|c|c|c|}
    \hline
    \textbf{}  & \textbf{2} & \textbf{4} & \textbf{8} & \textbf{16} \\
    \hline
    EF21  & 79.70\% & 79.59\% & 79.23\% & 77.97\% \\
    \hline
    Kimad & 79.34\% & 79.75\% & 78.74\% & 77.97\% \\
    \hline
    \end{tabular}
    \caption{Top5 accuracy across varying $M$. $T_{comm}=1s$.}
    \label{tb:Accuracy}
\end{table}
\section{Limitations And Future Work}
Kimad introduces a user-defined hyperparameter $t$, which is a trade-off between per-step time and accuracy and can also be adjusted dynamically. The learning rate can also be adjusted layer-wise
Besides, our work is not yet a fully implemented system. As the current experiments are simulation based, thus the implementation of monitor is trivial. We value the importance to integrate SOTA monitoring method to a complete work in the future. We can generalize the idea from splitting models to layers to blocks, where one block may contain many small layers. The computation overhead of Kimad+ is non-negligible and can be overlapped with communication. Moreover, LLM-targeted compression such as CocktailSGD~\cite{CocktailSGD} can also be considered.
\section{Conclusion}
We proposed Kimad, a bandwidth-aware gradient compression framework that comes with extended EF21. Kimad adapts the compression ratio based on the bandwidth and model characteristics; namely, each worker determines its local compression ratio considering its available bandwidth and time budget, and this ratio can be allocated to different layers in a non-uniform manner based on layer-wise sensitivity. We validated that Kimad can preserve the same convergence of fixed-ratio compression while saving communication time.

\subsection*{Acknowledgments}
This publication is based upon work supported by King Abdullah University of Science and Technology Research Funding (KRF) under Award No. ORA-2021-CRG9-4382.
For computer time, this research used the resources of the Supercomputing Laboratory at KAUST.

\bibliographystyle{ACM-Reference-Format}
\bibliography{main}

\clearpage
\appendix
\onecolumn
\section{Problem Formulation}
\label{app:formulation}
Algorithm~\ref{alg:META} is a generic form for solving problem  \eqref{eq:main} which can be used as a baseline.
\begin{algorithm}[h]
	\centering
	\caption{A Generic Distributed Training Meta-Algorithm}\label{alg:META}
	\begin{algorithmic}[1]
		\State {\bf Input:} loss $\ell$; weights $w_m$,  datasets $\cD_m$ and algorithms $A_m^{\rm update}$ for $m\in [M]$; learning rate schedule $\{\gamma^k\}$ for iterations $k\geq 0$
		\State Initialize with model $x^0\in \R^d$ on the server
		\For{$k=0,1, 2, \dots $}
			\State The server broadcasts model $x^k$ to all machines $m \in [M]$
			\State Each machine $m \in [M]$ uses algorithm $A_m$ to compute the update
\[u_m^k = A_m^{\rm update} \left(x^k, \ell, \cD_m \right) \in \R^d\]

\State Each machine $m\in [M]$ uploads the update $u_m^k$ to the server
			\State The server  updates the model via
\[x^{k+1} = x^k - \gamma^k \sum_{m=1}^M w_m u_m^k ,\]
where $\gamma^k>0$ is a learning rate
		\EndFor
	\end{algorithmic}
\end{algorithm}

Here are some canonical examples:
\begin{itemize}[left=0pt,labelsep=10pt]
	\item If $A_m^{\rm update}$ performs one step of gradient descent with respect to function $f_m$, i.e., $$u_m^k = \alpha^k \nabla  \ExpD{\xi\sim \cD_m}{\ell(x^k,\xi)} = \alpha^k \nabla f_m(x^k),$$ where $\alpha_m^k$ is a step size, then Algorithm~\ref{alg:META} becomes gradient descent (with step size $\gamma^k\alpha^k$) for solving problem~\eqref{eq:main}. If $A_m^{\rm update}$ applies multiple steps of gradient descent instead,
	then Algorithm~\ref{alg:META} becomes local gradient descent~\cite{localGD, localSGD-AISTATS2020}.
	
		\item If $A_m^{\rm update}$ performs one step of stochastic gradient descent with respect to function $f_m$, i.e.,  $$u_m^k = \frac{1}{b} \sum_{j=1}^b \nabla \ell(x^k,\xi_j^k), \qquad \text{where} \qquad \xi_1^k,\dots,\xi_b^k \sim \cD_m,$$ then Algorithm~\ref{alg:META} becomes a variant of mini-batch stochastic gradient descent for solving problem~\eqref{eq:main}. If $A_m^{\rm update}$ applies multiple steps of stochastic gradient descent instead,
	then Algorithm~\ref{alg:META} becomes local stochastic gradient descent~\cite{localSGD-AISTATS2020}.

\end{itemize}
In practice, not all workers will participate in every epoch's training. There are many worker sampling algorithms proposed~\cite{
Li2019OnTC, Sahu2018FederatedOI} to speed up the training. However, these algorithms can also introduce bias and have various behaviors on different tasks. In this work, we consider the situation of full participation of workers to avoid the influence of worker sampling.
\begin{enumerate}
\item The server broadcasts model $x^k$ to all workers $m \in [M]$;
\item Each machine $m \in [M]$ computes update
\[\hat{u}_m^k = A_m \left(x^k, \ell, \cD_m \right)\]
via some algorithm $A_m$ and uploads the update to the server;
\item The server aggregates the updates and updates the model via
\[x^{k+1} = x^k - \gamma^k \sum_{m=1}^M w_m \hat{u}_m^k ,\]
where $\gamma^k$ is a learning rate.
\end{enumerate}

\section{Kimad Algorithm with Explanation}
Algorithm \ref{alg:KimadFull} illustrates the Kimad algorithm with more details and comments.
\label{app:kimad}
\begin{algorithm}[H]
	\centering
		\caption{Kimad: Adaptive Gradient Compression with Bandwidth Awareness (Detailed)}
	\begin{algorithmic}[1]
		\State {\bf Input:} loss $\ell$; weights $w_m$, datasets $\cD_m$ and algorithms $A_m^{\rm update}$ for computing the model update on each machine  $m\in [M]$;  set of compressors $\Omega$; compressor-selection algorithm $A^{\rm compress}$ used by the server and the machines; model $x^0\in \R^d$ known by the server; initial model estimator  $\hat{x}^{-1}\in \R^d$  known by the machines and the server (for example, $\hat{x}^{-1}=0$ or $\hat{x}^{-1}=x^0$ are acceptable choices);  initial update estimators  $\hat{u}_{m}^{-1}\in \R^d$ for $m\in [M]$ known by the machines and the server (for example, $\hat{u}_{m}^{-1}=0$ for all $m\in [M]$ is an acceptable choice);   single round time budget $t>0$; learning rate schedule $\{\gamma^k\}$ for iterations $k\geq 0$

		\For{{\bf each communication round} $k=0,1, 2, \dots $}

\State The server estimates the broadcast/downlink bandwidth at communication round $k$; let the estimate be $B^k$

  \State The server chooses a compressor $\cC^k \in \Omega$ for compressing the difference $x^k - \hat{x}^{k-1}$ via  algorithm $A^{\rm compress}$
\[\cC^k = A^{\rm compress}(\Omega, x^k, \hat{x}^{k-1}, B^k, t)\]
(The algorithm $A^{\rm compress}$ aims to choose the compressor from $\Omega$ suffering minimal error when compressing the difference $x^k- \hat{x}^{k-1}$, subject to the constraint that the compressed message should take at most $t$ seconds to broadcast to the machines given the broadcast bandwidth estimate $B^k$)

\State  The server updates the model estimator to $$\hat{x}^{k} = \hat{x}^{k-1} + \cC^k( x^k - \hat{x}^{k-1})$$

\State  The server broadcasts the compressed vector $\cC^k( x^k - \hat{x}^{k-1})$
to all machines $m \in [M]$

\For{{\bf each machine} $m=1, 2, \dots, M $ {\bf in parallel}}			
            \State  Update the model estimator to $$\hat{x}^{k} = \hat{x}^{k-1} + \cC^k( x^k - \hat{x}^{k-1})$$ using the previously stored estimator $\hat{x}^{k-1}$ and the received message

			\State Use algorithm $A_m^{\rm update}$ to compute the update
\[u_m^k = A_m^{\rm update} \left(\hat{x}^k, \ell, \cD_m \right) \in \R^d\]

            \State Estimate the uplink bandwidth of machine $m$ at communication round $k$; let the estimate be $B_m^k$

\State Choose a compressor $\cC_{m}^k \in \Omega$ for compressing the difference $u_m^k - \hat{u}_m^{k-1}$ via  algorithm $A^{\rm compress}$
\[\cC_m^k = A^{\rm compress}(\Omega, u_m^k, \hat{u}_m^{k-1}, B_{m}^k, t)\]
(The algorithm $A^{\rm compress}$ aims to choose the compressor from $\Omega$ suffering minimal error when compressing the difference $u_m^k - \hat{u}_m^{k-1}$, subject to the constraint that the compressed message should take at most $t$ seconds to upload to the server given the uplink bandwidth estimate $B_m^k$)

\State  Upload the compressed vector {\color{black}$\cC_{m}^k(u_m^k-\hat{u}_m^{k-1} )$}
 to the server
		\EndFor
\State The server updates all update estimators to $$\hat{u}_m^k = \hat{u}_{m}^{k-1} + \cC_{m}^k(u_m^k-\hat{u}_m^{k-1} ), \qquad m\in [M]$$
\State The server updates the model via
\[x^{k+1} = x^k - \gamma^k \sum_{m=1}^M w_m \hat{u}_m^k ,\]
where $\gamma^k>0$ is a learning rate
		\EndFor
	\end{algorithmic}
\label{alg:KimadFull}
\end{algorithm}

\section{Kimad+ Dynamic Programming}
Algorithm \ref{alg:dp} lists the dynamic programming algorithm to optimize the layer-wise compression ratio allocation to minimize the compression error.
\label{app:dp}
\begin{algorithm}[H]
	\centering
		\caption{A dynamic programming algorithm to allocate compression ratio across layers}\label{alg:layerwise}
	\begin{algorithmic}[1]
             \State {\bf Input:} Model Layers $L_i$, accumulated gradients $G_i$, possible compression parameters $C = \{c^1, c^2, \ldots, c^k\}$, model compression budget $\mathcal{E}_{\max}$, discretization factor $D$
             \State N = number of layers
             \State Compute $\mathcal{E}_{\max}$ for the default compression parameters $C_i^d$
             \State Compute discretization step $\mathcal{E}_{\max} / D$.
             \State Costs matrix $N \times |C|$ where position $i, j$ has a value of the size of layer $i$ compressed with compression parameter $c^j$.
             \State Errors matrix $N \times |C|$ where position $i, j$ has a value of the discretized $L_2$ of the compression error when the accumulated gradients of layer $i$ are compressed with parameter $c^j$.
             \State $DP$ matrix $N \times (D + 1)$ filled with $\infty$ values.
             \State $PD$ matrix $N \times (D + 1)$.
             \For{$c \in C$}
                \State $DP[1][Costs[1][c]] = Errors[1][c]$
                \State $PD[1][Costs[1][c]] = c$
             \EndFor
             \State // Dynamic programming algorithm
             \For{Layer $l_i:= 2 .. N$}
                \For {$c_i \in C$}
                    \For {$cost_i:= Costs[l_i][c_i] .. D$}
                        \State $t = DP[l_i - 1][cost_i - Costs[l_i][c_i]] + Errors[l_i][c_i]$
                        \If{$t < DP[l_i][e_i]$}
                             \State $DP[l_i][e_i] = t$
                            \State $PD[l_i][e_i] = c_i$
                        \EndIf
                    \EndFor
                \EndFor
            \EndFor
             \State $cost_{min} = argmin(DP[N])$
             \State // Reconstruction of the optimal parameters
             \For{$l_i = N .. 1$}
                \State $result[l_i] = PD[l_i][cost_{min}]$
                \State $cost_{min} = cost_{min} - Costs[l_i][result[l_i]]$
             \EndFor
             \State \textbf{return} $result$
	\end{algorithmic}
\label{alg:dp}
\end{algorithm}

\section{Assumptions and Basic Identities}
\label{app:assumptions}
\subsection{Layer Smoothness}
\label{app:layer_smooth}
\begin{assumption}[Layer smoothness] \label{ass:L_i} There exist constants $L_1,\dots,L_{\nlay}>0$ such that
\[f(x+s) \leq f(x) + \inner{\nabla f(x)}{s} + \frac{1}{2}\sum_{i=1}^{\nlay} L_i \norm{s_i}^2\]
holds for all $x\in \R^d$ and $s=(s_1,\dots,s_{\nlay}) \in \R^d$.
\end{assumption}
\subsection{Global Smoothness}
\label{app:global_smooth}
\begin{assumption}[Global smoothness] \label{ass:L} There exists a constant $L>0$ such that
\begin{equation}\label{eq:L-Lipschitz-gradient} \norm{\nabla f(x+s) - \nabla f(x)} \leq  L \norm{s}\end{equation}
holds for all $x,s\in \R^d$.
\end{assumption}
\label{app:inequality}
\subsection{Definition of $\bar{G}^{k}$}
Choose any $\delta_1,\dots,\delta_{\nlay}>0$ and define
\begin{equation}\label{eq:def-G^k}
\bar{G}^{k} \eqdef \Exp{G^{k}}, \; G^{k} \eqdef  \sum_{i=1}^{\nlay} \delta_i G_i^{k}, \; G_i^k \eqdef \sqnorm{\hat{u}_i^k - \nabla_i f(x^{k})}.
\end{equation}
\subsection{Young's inequality} For any $x,y\in \R^d$ and any $\zeta>0$ we have \begin{equation}\label{eq:Young} \sqnorm{x+y}\leq (1+\zeta)\sqnorm{x} + (1+\zeta^{-1})\sqnorm{y}.\end{equation}

\section{Technical Lemmas}
\label{app:lemma}
\begin{lemma}[Technical identity]\label{lem:relation}
Let $x,\hat{u}\in \R^d$ and  \begin{equation}\label{eq:b79gfd-9898df-step}x^+ \eqdef x - \gamma \hat{u},\end{equation} where $\gamma >0$. Then for any $M>0$ and any $h\in \R^d$ we have the identity
	\begin{align}
 \inner{h}{x^+-x}
	+ \frac{M}{2}\norm{x^+-x}^2
& =  - \frac{\gamma}{2} \norm{h}^2 - \left(\frac{1}{2\gamma} - \frac{M}{2}\right) \norm{x^+ -x}^2  	+ \frac{\gamma}{2}\norm{\hat{u} - h}^2.\label{eq:relation_0}
	\end{align}
\end{lemma}

Let $x^k = (x_1^k,\dots,x_{\nlay}^k)$. The lemma below holds for any algorithm of the form
\[x^{k+1}_i = x^k_i - \gamma_i^k \hat{u}_i^k, \qquad i=1,\dots,\nlay,\]
where $\gamma_1^k,\dots,\gamma_{\nlay}^k>0$ are stepsizes.  Therefore, it also holds for Algorithm~\eqref{eq:EF21-1}--\eqref{eq:EF21-3}.

\begin{lemma}[Descent]\label{lem:98u09ufd_897} If Assumption~\ref{ass:L_i} holds, then
\begin{eqnarray*}
f(x^{k+1})  \leq
f(x^k) + \sum_{i=1}^{\nlay} \left(   - \frac{\gamma_i^k}{2} \norm{\nabla_i f(x^k)}^2 - \left(\frac{1}{2\gamma_i^k} - \frac{L_i}{2}\right) \norm{x_i^{k+1}-x_i^{k}}^2  + \frac{\gamma_i^k}{2}\norm{\hat{u}_i^k - \nabla_i f(x^k)}^2 \right).
\end{eqnarray*}

\end{lemma}
\begin{proof}
By applying Lemma~\ref{lem:relation} with $d\leftarrow d_i$, $x^+ \leftarrow x_i^{k+1}$, $x\leftarrow x_i^k$, $\gamma \leftarrow \gamma_i^k$, $\hat{u}\leftarrow \hat{u}_i^k$, $h\leftarrow \nabla_i f(x^k)$ and $M\leftarrow L_i$, we get
{\small
	\begin{eqnarray}
 \inner{\nabla_i f(x^k)}{x_i^{k+1} - x_i^{k}}
	+ \frac{L_i}{2}\norm{x_i^{k+1}-x_i^{k}}^2
 &=&  - \frac{\gamma_i^k}{2} \norm{\nabla_i f(x^k)}^2 - \left(\frac{1}{2\gamma_i^k} - \frac{L_i}{2}\right) \norm{x_i^{k+1}-x_i^{k}}^2  \notag  \\
 && \quad	+ \frac{\gamma_i^k}{2}\norm{\hat{u}_i^k - \nabla_i f(x^k)}^2.\label{eq:relation}
	\end{eqnarray}
	}
Using Assumption~\ref{ass:L_i}, with $x\leftarrow x^k = (x_1^k,\dots,x_{\nlay}^k)$ and 	$s\leftarrow x^{k+1}-x^k = (x_1^{k+1}-x_1^{k}, \dots, x_{\nlay}^{k+1}-x_{\nlay}^{k}) = (-\gamma_1^k \hat{u}_1^k, \dots, -\gamma_{\nlay}^k \hat{u}_{\nlay}^k)$, we get
\begin{eqnarray*}
f(x^{k+1}) & \leq & f(x^k) + \sum_{i=1}^{\nlay} \inner{\nabla_i f(x^k)}{ x_i^{k+1}-x_i^{k} } + \frac{1}{2}\sum_{i=1}^{\nlay} L_i \norm{x_i^{k+1}-x_i^{k}}^2 \\
&=& f(x^k) + \sum_{i=1}^{\nlay} \left( \inner{\nabla_i f(x^k)}{ x_i^{k+1}-x_i^{k} } + \frac{L_i}{2}  \norm{x_i^{k+1}-x_i^{k}}^2 \right)\\
&\overset{\eqref{eq:relation}}{=}&
f(x^k) + \sum_{i=1}^{\nlay} \left(   - \frac{\gamma_i^k}{2} \norm{\nabla_i f(x^k)}^2 - \left(\frac{1}{2\gamma_i^k} - \frac{L_i}{2}\right) \norm{x_i^{k+1}-x_i^{k}}^2  + \frac{\gamma_i^k}{2}\norm{\hat{u}_i^k - \nabla_i f(x^k)}^2 \right).
\end{eqnarray*}

\end{proof}

Our next lemma is specific to Algorithm~\eqref{eq:EF21-1}--\eqref{eq:EF21-3}.
	
	\begin{lemma}[3PC inequality] \label{lem:theta-beta} Choose any $i\in [\nlay]$.  Let $\cC_i^k \in \mathbb{C}^{d_i}(\alpha_i)$ for $0<\alpha_i \leq 1$. Define
	$G_i^k \eqdef  \sqnorm{ \hat{u}_i^k - \nabla_i f(x^k) } $ and $W^k \eqdef \{\hat{u}_i^k, \dots, \hat{u}_{\nlay}^k, x^k, x^{k+1}\}$. 		Then		\begin{equation}\label{eq:90y0yfhdf} \Exp{ G_i^{k+1} \;|\; W^k} \leq (1-\theta_i)   G_i^k + \beta_i  \sqnorm{\nabla_i f(x^{k+1}) - \nabla_i f(x^k)} ,
		\end{equation}
		where
		\begin{equation}\label{eq:theta-beta-def}\theta_i \eqdef 1- (1- \alpha_i )(1+\zeta_i), \qquad \text{and} \qquad \beta_i \eqdef (1- \alpha_i ) \left(1+ \zeta_i^{-1} \right)
		\end{equation}
		and  $\zeta_i$ is any positive number.

		\end{lemma}

 	\begin{proof}
 		\begin{eqnarray*}
 			\Exp{ G_i^{k+1} \;|\; W^k} & = & \Exp{  \sqnorm{\hat{u}_i^{k+1} - \nabla_i f(x^{k+1})}  \;|\; W^k}	 \\
 			&\overset{\eqref{eq:EF21-2} }{=}& \Exp{  \sqnorm{\hat{u}_i^k + \cC_i^k ( \nabla_i f(x^{k+1}) - \hat{u}_i^k) - \nabla_i f(x^{k+1})}  \;|\; W^k}	 \\
 			& \leq &  (1-\alpha_i) \sqnorm{\nabla_i f(x^{k+1}) - \hat{u}_i^k} \\
		& =&  (1-\alpha_i) \sqnorm{ \nabla_i f(x^{k}) - \hat{u}_i^k + \nabla_i f(x^{k+1}) - \nabla_i f(x^{k}) } \\
			&\overset{\eqref{eq:Young}}{\leq} & (1-\alpha_i)  (1+ \zeta_i) \sqnorm{\nabla_i f(x^{k}) - \hat{u}_i^k}   + (1-\alpha_i)  \left(1+\zeta_i^{-1}\right) \sqnorm{\nabla_i f(x^{k+1}) - \nabla_i f(x^k)} , \end{eqnarray*}
			where the first inequality holds since $\cC_i^k \in \mathbb{C}^{d_i}(\alpha_i)$, and in the last step we have applied Young's inequality.			
 	\end{proof}

\section{Proof of Theorem~\ref{thm:main}}
\label{app:proof}
\begin{proof} We proceed in three steps:		
	\paragraph{\bf STEP 1.}
First, we note that Lemma~\ref{lem:theta-beta} says that			\begin{equation}\label{eq:n89fg9d08hfbdi_8f}
				\Exp{\sqnorm{\hat{u}_i^{k+1} -  \nabla_i f(x^{k+1})}\mid W^k } \overset{\eqref{eq:90y0yfhdf}}{\leq} (1 - \theta_i)\sqnorm{\hat{u}_i^k - \nabla_i f(x^k)} + \beta_i  \sqnorm {\nabla_i f(x^{k+1}) - \nabla_i f(x^{k})}.
			\end{equation}
			
			 Adding inequalities \eqref{eq:n89fg9d08hfbdi_8f} over $i \in [\nlay]$ and recalling that $G_i^k \eqdef \sqnorm{\hat{u}_i^k - \nabla_i f(x^k)}$, we get
		
			\begin{eqnarray}
				\Exp{G^{k+1} \mid W^k } &\overset{\eqref{eq:def-G^k}}{=}&
				\Exp{ \sum_{i=1}^{\nlay} \delta_i G_i^{k+1} \mid W^k } 		
\notag				\\
	&\overset{\eqref{eq:def-G^k}}{=}&			\Exp{ \sum_{i=1}^{\nlay} \delta_i \sqnorm{ \hat{u}_i^{k+1} - \nabla_i f(x^{k+1})} \mid W^k} \notag \\
	&=&			\sum_{i=1}^{\nlay} \delta_i \Exp{ \sqnorm{ \hat{u}_i^{k+1} - \nabla_i f(x^{k+1})} \mid W^k} \notag \\	
&\overset{\eqref{eq:n89fg9d08hfbdi_8f}}{\leq}&  \sum_{i=1}^{\nlay} \delta_i \left( \left(1 - \theta_i \right) \sqnorm{\hat{u}_i^k - \nabla_i f(x^{k})} +  \beta_i  \sqnorm {\nabla_i f(x^{k+1}) - \nabla_i f(x^{k})} \right) \notag \\
&=&  \sum_{i=1}^{\nlay} \delta_i  \left(1 - \theta_i \right) \sqnorm{\hat{u}_i^k - \nabla_i f(x^{k})} + \sum_{i=1}^{\nlay} \delta_i  \beta_i  \sqnorm {\nabla_i f(x^{k+1}) - \nabla_i f(x^{k})}\notag \\
				&\overset{\eqref{eq:def-G^k}}{=}&  \roundbrack{1 - \min_i \theta_i}G^k+  \sum_{i=1}^{\nlay} \delta_i \beta_i \sqnorm {\nabla_i f(x^{k+1}) - \nabla_i f(x^{k})} \notag \\
				&\le& \roundbrack{1 - \min_i \theta_i}G^k+  \left( \max_i \delta_i \beta_i\right) \sum_{i=1}^{\nlay} \sqnorm {\nabla_i f(x^{k+1}) - \nabla_i f(x^{k})}  \notag \\
& = & \roundbrack{1 - \min_i \theta_i}G^k+  \left( \max_i \delta_i \beta_i\right)  \sqnorm {\nabla f(x^{k+1}) - \nabla f(x^{k})}	\notag \\
&\overset{\eqref{eq:L-Lipschitz-gradient}}{\leq} &\roundbrack{1 - \min_i \theta_i}G^k+  \left( \max_i \delta_i \beta_i\right) L^2 \sqnorm {x^{k+1} - x^{k}}. \label{eq:jbiu-9u0df9}
			\end{eqnarray}
			Using the Tower property  in \eqref{eq:jbiu-9u0df9}, we proceed to
			\begin{eqnarray}\label{eq:main_recursion_distrib}
				\Exp{G^{k+1}} = \Exp{\Exp{G^{k+1} \mid W^k}} \overset{\eqref{eq:jbiu-9u0df9}}{\leq} \left(1 - \theta \right)\Exp{G^k}+  \left(\max_i \delta_i \beta_i \right)L^2   \Exp{ \sqnorm{x^{k+1} - x^k}}.
			\end{eqnarray}

	\paragraph{\bf STEP 2.}
 Next, using Lemma~\ref{lem:98u09ufd_897}, we obtain the bound
			\begin{eqnarray}
				f(x^{k+1}) &\leq& f(x^k) + \sum_{i=1}^{\nlay} \left(   - \frac{\gamma_i^k}{2} \norm{\nabla_i f(x^k)}^2 - \left(\frac{1}{2\gamma_i^k} - \frac{L_i}{2}\right) \norm{x_i^{k+1}-x_i^{k}}^2  + \frac{\gamma_i^k}{2}\norm{\hat{u}_i^k - \nabla_i f(x^k)}^2 \right).\notag \\
				& = &
				f(x^{k})- \sum_{i=1}^{\nlay} \frac{\gamma_i^k}{2}\sqnorm{\nabla_i f(x^{k})}-\sum_{i=1}^{\nlay} \left(\frac{1}{2 \gamma_i^k}-\frac{L_i}{2}\right)\sqnorm{x_i^{k+1}-x_i^{k}} \notag \\
				&& \qquad +\sum_{i=1}^{\nlay} \frac{\gamma_i^k}{2 \delta_i} \delta_i \norm{\hat{u}_i^k - \nabla_i f(x^k)}^2 \notag \\
				&\overset{\eqref{eq:def-G^k}}{\leq} &
		f(x^{k})- \sum_{i=1}^{\nlay} \frac{\gamma_i^k}{2}\sqnorm{\nabla_i f(x^{k})}-\sum_{i=1}^{\nlay} \left(\frac{1}{2 \gamma_i^k}-\frac{L_i}{2}\right)\sqnorm{x_i^{k+1}-x_i^{k}}  + \left(\max_i \frac{\gamma_i^k}{2 \delta_i} \right)  G^k.			\label{eq:aux_smooth_Lemma_distrib}
			\end{eqnarray}

			Subtracting $f^{\inf}$ from both sides of \eqref{eq:aux_smooth_Lemma_distrib} and taking expectation, we get
			\begin{eqnarray}
				\Exp{f(x^{k+1})-f^{\inf}} &\leq& \Exp{f(x^{k})-f^{\inf}}- \sum_{i=1}^{\nlay} \frac{\gamma_i^k}{2}\Exp{\sqnorm{\nabla_i f(x^{k})} } \notag \\
				&& \quad -\sum_{i=1}^{\nlay} \left(\frac{1}{2 \gamma_i^k}-\frac{L_i}{2}\right)\Exp{\sqnorm{x_i^{k+1}-x_i^{k}}} +  \left(\max_i \frac{\gamma_i^k}{2 \delta_i} \right) \Exp{ G^k } \notag \\
	&=&			
\Exp{f(x^{k})-f^{\inf}}- \frac{\gamma}{2}\sum_{i=1}^{\nlay} w_i \Exp{\sqnorm{\nabla_i f(x^{k})}} \notag \\
				&& \quad -\sum_{i=1}^{\nlay} \left(\frac{1}{2 \gamma w_i}-\frac{L_i}{2}\right)\Exp{\sqnorm{x_i^{k+1}-x_i^{k}}} +  \frac{\gamma}{2}\left(\max_i \frac{w_i}{\delta_i} \right) \Exp{ G^k } ,
\label{eq:func_diff_distrib}						\end{eqnarray}		
where we used the fact that $\gamma_i^k \equiv \gamma w_i$.

	\paragraph{\bf STEP 3: COMBINING PREVIOUS STEPS.}	 	
Due to the restriction $(1- \alpha_i )(1+\zeta_i) < 1$, which holds for all $i\in [\nlay]$, we know that $\theta_i>0$ for all $i$, and therefore, $\theta = \min_i \theta_i$ is also positive. This detail plays a role in what follows.

Let $\Delta^{k} \eqdef \Exp{f(x^{k})-f^{\inf}}$, $\bar{G}^{k} \eqdef \Exp{G^k }$, $R_i^k \eqdef \Exp{\sqnorm{x_i^{k+1}-x_i^k}}$ and
	$R^k \eqdef \sum_{i=1}^{\nlay} R_i^k = \Exp{\sqnorm{x^{k+1}-x^k}}.$
By adding \eqref{eq:func_diff_distrib} to the $\gamma\frac{\left(\max_i \frac{w_i}{\delta_i} \right)}{2 \theta}$ multiple of \eqref{eq:main_recursion_distrib}, we obtain
			\begin{eqnarray*}
				\Delta^{k+1}+\frac{\gamma\left(\max_i \frac{w_i}{\delta_i} \right)}{2 \theta} \bar{G}^{k+1} &\leq& \Delta^{k} - \frac{\gamma}{2}\sum_{i=1}^{\nlay} w_i \Exp{\sqnorm{\nabla_i f(x^{k})} } -\sum_{i=1}^{\nlay} \left(\frac{1}{2 \gamma w_i}-\frac{L_i}{2}\right) R_i^k\notag \\
				&&\quad +\frac{\gamma}{2} \left(\max_i \frac{w_i}{\delta_i} \right) \bar{G}^k +\frac{\gamma\left(\max_i \frac{w_i}{\delta_i} \right)}{2 \theta}\left(\roundbrack{1 - \theta} \bar{G}^k +  \left( \max_i \delta_i \beta_i\right) L^2 R^k	\right) \\
				&=&\Delta^{k}+\frac{\gamma\left(\max_i \frac{w_i}{\delta_i} \right)}{2\theta} \bar{G}^{k}-\frac{\gamma}{2}\sum_{i=1}^{\nlay} w_i \Exp{\sqnorm{\nabla_i f(x^{k})}} \notag \\
				&& \quad -\sum_{i=1}^{\nlay} \left(\frac{1}{2\gamma w_i} -\frac{L_i}{2} - \frac{\gamma\left(\max_i \frac{w_i}{\delta_i} \right)}{2\theta} \left(\max_i \delta_i \beta_i  \right) L^2 \right) R_i^k \\
				& \leq& \Delta^{k}+\frac{\gamma\left(\max_i \frac{w_i}{\delta_i} \right)}{2\theta} \bar{G}^{k} -\frac{\gamma}{2}\sum_{i=1}^{\nlay} w_i \Exp{\sqnorm{\nabla_i f(x^{k})}}.
			\end{eqnarray*}
			The last inequality follows from the bound \eqref{eq:stepsize_main}.	By summing up inequalities for $k =0, \ldots, K-1,$ we get
			$$
			0 \leq \Delta^{K}+\frac{\gamma\left(\max_i \frac{w_i}{\delta_i} \right)}{2 \theta} \bar{G}^{K} \leq \Delta^{0}+\frac{\gamma\left(\max_i \frac{w_i}{\delta_i} \right)}{2 \theta} \bar{G}^{0}-\frac{\gamma}{2} \sum_{k=0}^{K-1} \left( \sum_{i=1}^{\nlay} w_i \Exp{\sqnorm{\nabla_i f(x^{k})}}\right).
			$$
			Multiplying both sides by $\frac{2}{\gamma K}$, after rearranging we get
			$$
			\frac{1}{K} \sum_{k=0}^{K-1}  \left( \sum_{i=1}^{\nlay} w_i \Exp{\sqnorm{\nabla_i f(x^{k})}}\right) \leq \frac{2 \Delta^{0}}{\gamma K} + \frac{\left(\max_i \frac{w_i}{\delta_i} \right)\bar{G}^0}{\theta K}.
			$$
				\end{proof}
\end{document}